\def\year{2022}\relax
\documentclass[letterpaper]{article} 
\usepackage{aaai22}  
\usepackage{times}  
\usepackage{helvet}  
\usepackage{courier}  
\usepackage[hyphens]{url}  
\usepackage{graphicx} 
\usepackage{natbib}  
\usepackage{caption} 
\DeclareCaptionStyle{ruled}{labelfont=normalfont,labelsep=colon,strut=off} 
\usepackage{algorithm}
\usepackage{algorithmic}

\usepackage{newfloat}
\usepackage{listings}
\usepackage{multirow}
\usepackage{amsmath}
\usepackage{amsthm}

\floatstyle{ruled}
\newfloat{listing}{tb}{lst}{}
\floatname{listing}{Listing}
\newtheorem{userdefined}{Theorem}
\title{Sampling from Pre-Images to Learn Heuristic Functions for Classical Planning}
\author {
    Stefan O'Toole\textsuperscript{\rm 1},
    Miquel Ramirez\textsuperscript{\rm 2},
    Nir Lipovetzky \textsuperscript{\rm 1},
    Adrian R. Pearce \textsuperscript{\rm 1}
}
\affiliations {
    \textsuperscript{\rm 1} Computing and Information Systems, University of Melbourne, Australia\\
    \textsuperscript{\rm 2} Electrical and Electronic Engineering, University of Melbourne, Australia\\
    stefan@student.unimelb.edu.au, \{miquel.ramirez, nir.lipovetzky, adrianrp\}@unimelb.edu.au
}

\date{}

\begin{document}

\maketitle

\begin{abstract}
We introduce a new algorithm, \emph{Regression based Supervised Learning} (RSL), for learning per instance Neural Network (NN) defined heuristic functions for classical planning problems. RSL uses regression to select relevant sets of states at a range of different distances from the goal. RSL then formulates a Supervised Learning problem to obtain the parameters that define the NN heuristic, using the selected states labeled with exact or estimated distances to goal states. Our experimental study shows that RSL outperforms, in terms of coverage, previous classical planning NN heuristics functions while requiring two orders of magnitude less training time.
\end{abstract}

\section{Introduction}
Heuristics for automated planning can be formulated following a number of approaches. 
Heuristics such as Fast-Forward (FF)~\cite{hoffmann2001ff}, $h\textsuperscript{max}$ and $h\textsuperscript{add}$~\cite{bonet1997robust} follow from analyzing the structure of
many planning instances, and coming up with a mathematical framework to  automatically compute
functions that capture important structural information about instances from the symbolic 
descriptions of causal laws (actions) and domain constraints~\cite{helmert2006fast}. Alternatively, Pattern database~\cite{edelkamp2014planning} and \emph{merge-and-shrink}
~\cite{helmert2007flexible} heuristics are defined as generic functions that evaluate states
by projecting them onto many smaller sub-problems, that are solved optimally, and combining
their solutions in some specific way. These are chosen on a per-instance basis,
and typically involve solving a discrete optimization problem to select which projection
is deemed to be most informative according to some suitably defined criterion. 
A third approach has attracted attention recently, where heuristic functions are searched for in
a family of functions described by a Neural Network
(NN)~\cite{ferber:2021:prl,ferber2020neural,shen2020learning}. These efforts are mainly driven by
the suggestive results in Computer Vision and Reinforcement
Learning of novel, general, and scalable stochastic algorithms for convex 
optimization to select NN parameters that
minimize a suitably defined notion of \emph{empirical risk}~\cite{kingma2014adam,goodfellow2016deep,hardtrecht}.
Currently, the most successful methods for learning NN heuristics for classical planning problems require vast amounts of computational resources for training and are usually outperformed by heuristics 
that do not 
require any offline training time~\cite{ferber:2021:prl}.

In this work, we introduce the \emph{Regression based Supervised Learning} (RSL) algorithm in order to learn per instance NN defined heuristic functions. Like other methods~\cite{ferber:2021:prl,yu2020learning} do,
RSL selects a set of \emph{regressions}, trajectories of \emph{sets of states}, or \emph{pre-images},
found via the
application of well-known and efficient pre-imaging
operators~\cite{rintanen2008regression} that rely on symbolic action descriptions. These 
trajectories found along a given regression, always start from the set of goal states of an instance, and then
training states are sampled from each set along the trajectory.
Our method takes many samples from each pre-image found in a regression, instead of performing many trajectories or longer regressions to increase the number of training states for the NN,
using the observed goal distances for each pre-image to label the sampled
states.

Through a sensitivity analysis over RSL's hyper-parameters we explore the impact of its key mechanisms. We also benchmark the NN heuristics learnt by RSL against existing methods to provide insight into the ability of RSL to learn effective heuristics.  
The paper's contributions are threefold: (1) introducing a new method for training NN defined heuristic functions, (2) introducing a new novelty measure for a regression based search, and (3) an experimental analysis of RSL's hyper-parameters and components.

\section{Background}
\textbf{Classical Planning}. We consider classical planning problems as defined by the STRIPS formulation~\cite{fikes1971strips}. STRIPS defines a planning problem as the tuple $\Pi = \langle F,O,I, G\rangle$, where $F$ is a set of atoms, $O$ is a set of actions, $I\subseteq F$ is the initial state and $G\subseteq F$ is the goal set. Each action $a \in O$ is represented by the tuple $\langle Add(a), Del(a), Pre(a) \rangle$ which are each a set of atoms over $F$. $Add(a)$ and $Del(a)$ describe the atoms that are added and removed from the state respectively, and $Pre(a)$ describes the atoms that must be true in a state in order to apply action $a$. The STRIPS problem $\Pi$ implicitly represents in a compact form a deterministic transition system~\cite{geffner2013concise}. The classical planning state-transition model for progression is $\mathcal{S}(\Pi) = \langle S, s_0, S_G, A, f, c\rangle$, where $S \subseteq 2^F$, $s_0$ is the initial state $I$, $S_G$ is the set of goal states described as the set $\{ s \  | \ s \supseteq G, s \in S\}$, the actions $a \in A(s)$ are the actions in $O$ that are applicable in $s$, that is $Pre(a) \subseteq s$, $f$ is the transition function where for action $a$ and state $s$, the resulting state is $s' = f(a,s) = (s\setminus Del(a)) \cup Add(a)$, and finally $c(a,s)$ is the cost of selecting the transition out of $s$ via action $a$. The problems we consider in this paper all have a uniform cost $c(a,s)=1$.  A solution for a classical planning problem is a sequence of actions $a_0$,\ldots,$a_n$ that select transitions connecting the initial state $s_0$ to a state within the goal set $S_G$. The progression state-transition model can be used to find a solution through a forward search for a goal state from $s_0$.

\textbf{Heuristics}. The objective of the heuristics introduced in this paper is to approximate the optimal value function $V^*$ as described through the Bellman optimality equation~\cite{bellman:57},
\begin{align}
	\label{eq:Bellman_Equation}
	V^{*}(s) = \min_{a \in A(s)} \bigg[ c(a,s) + V^*(f(a,s))\bigg]
\end{align}
for all states $s \not\in S_G$ and $V^{*}(s) = 0$ for $s \in S_G$. Similar to many of the existing heuristics such as $h\textsuperscript{FF}$ the heuristics introduced in this paper are not guaranteed to be lower or upper bounds on the optimal value function.

\textbf{Regression}. Planners can also search backwards from the goal through the regression state-transition model~\cite{geffner2013concise}. The regression state-transition  model is $R(\Pi) = \langle S, s_0, S_G, A, f, c\rangle$, where $s_0$ is the partially assigned state $G$, the goal set $S_G$ is the set $\{ s \ | \ s \subseteq I, s \in S\}$, $A(s)$ are the actions in $O$ that are relevant and consistent for the partial state $s$, that is $Add(a) \cap s \not = \emptyset$ and $Del(a) \cap s = \emptyset$, and the state transition function to pre-image state $s$ is $f(a,s) = (s \setminus Add(a)) \cup Pre(a)$. A key difference between the regression and progression state-transition models is that the regression model searches over partial truth assignments, which represent sets of states as opposed to complete truth-assignments. For the progression state-transition model every atom not in a state is false, while for the regression state-transition model the atoms not in a state are simply undefined. From this point forward we refer to such ``partial states'' as \emph{pre-images} and denote them as $x$ to distinguish them from complete truth-assignment state denoted as $s$. The regression model is required to search over pre-images $x \subseteq F$ as the initial pre-image $x_0$ coincides with the set of goal
states from the progression state-transition model $S_G$. Regression operators also compute \emph{weaker preconditions} so the only 
atoms being asserted in a pre-image $x$ are those in the precondition $Pre(a)$, while retracting
any commitments on the truth value of atoms in $Add(a)$. Conversely, the progression state model will always search over full states $s \in \mathcal{S}(\Pi)$ as $I$ prescribes the truth value of every
atom, and the progression transition function $f(a,s)$ provides an explicit mapping between states 
$s, s' \in \mathcal{S}(\Pi)$. 

\textbf{Supervised Learning}. In this paper we take advantage of the pre-images explored by a regression search in order to sample a set of training states labelled with approximations $\tilde{V}$ of $V^*$ (Equation~\ref{eq:Bellman_Equation}),  to learn a heuristic defined through a NN using a Supervised Learning (SL) algorithm.

SL is a common approach for creating predictors from data samples through empirical risk minimisation~\cite{hardtrecht} with a goal of generalising to out-of-sample data. That is, the objective of SL given a function class $\mathcal{H} \subseteq \mathcal{X}\to\mathcal{Y}$ and a set of data samples $(\bar{x}_1, y_1), (\bar{x}_2, y_2), \ldots, (\bar{x}_i, y_i)$ is defined as,
\begin{align}
\label{eq:empirical_risk_min}
    \min_{h \in \mathcal{H}}\frac{1}{n} \sum_{i=1}^n {\mathit{loss}}( h(\bar{x}_i), y_i )
\end{align}
In this paper, ${\cal H}$ is a NN, each $h \in {\cal H}$ is a unique set of NN parameters $\theta$,
training examples $\bar{x}_i$ are states $s$, and labels $y_i$ are cost-to-go estimates $\tilde{V}$ between $s$ and
$S_G$.

\textbf{Per-instance heuristics}. Ferber et al.~\shortcite{ferber:2021:prl} describe two different methods for learning heuristics for classical planning problems. The first is \emph{per domain} learning, where a heuristic is learnt that is applicable to any instance of a particular domain, where a domain defines a set of action schemas and predicates which are used to formulate instances $\Pi$. Our paper focuses on the second framework introduced by Ferber et al. (2021), which is to learn instance-based heuristics. That is, given some instance $\mathcal{S}(\Pi)$ with initial state $s_0$, we seek heuristics $h$ that apply to instances $\mathcal{S}(\Pi_1)$, $\mathcal{S}(\Pi_2)$, \ldots, $\mathcal{S}(\Pi_i)$ where all the instances $\mathcal{S}(\Pi_i)$ share all structural elements with $\mathcal{S}(\Pi)$ but initial states. That is, each $\mathcal{S}(\Pi_i)$ features a distinct initial state $s^i_0$, which is reachable from $s_0$ in $\mathcal{S}(\Pi)$. Our objective in this work is to efficiently, in terms of computation time, create a data sample of states labelled with $\tilde{V}$ such that an optimisation of the NN parameters $\theta$ using Equation~\ref{eq:empirical_risk_min} generalises to the population set of states $s^i_0$ reachable from $s_0$ .

\section{Related Work}
Ferber et al.~\shortcite{ferber2020neural} introduce a method, Teacher-based Supervised Learning (TSL), that constructs a training set out of states found in the trace of the execution of plans from a teacher planner. TSL performs 200 step random walks from the initial state of an instance and then uses a Greedy Best First Search (GBFS) with $h$\textsuperscript{FF} to search for a plan. The states along the solution path (if found) are labeled with an estimate of their distance from the goal according to the solution. TSL then performs Supervised Learning (SL) on the states and goal-distance estimates to learn a heuristic. 

Yu et al.~\cite{yu2020learning} introduced an alternative SL method for learning per-instance heuristic functions. Yu et al.'s algorithm (SING) performs backwards depth-first searches (DFS) with fixed node expansion budgets, to search from fully assigned randomly sampled goal states $s_g \in S_G$. SING then uses SL on states visited within the DFS labeled with the depth at which the states were visited for an estimate of their goal distance. 

The algorithm we introduce in this work differs from both TSL and SING in a number of ways. First, RSL samples training states through a regression over partially assigned states starting from the goal rather than a DFS starting at a fully assigned goal state (SING) or a forward search from the instance's initial state (TSL). Second, RSL additionally samples random states that are added to the training set. Third, for goal distance estimates TSL uses a teacher planner and SING uses the depth at which the state was visited while RSL uses the tightest upper bound found for the state's goal distance derived through the pre-images visited by a number of regressions. Last, as RSL searches and labels goal distance estimates over pre-images, it samples many different training states from each pre-image.

More recently Ferber et al.~\shortcite{ferber:2021:prl} introduced three new NN heuristics. Two of the approaches introduced, $h$\textsuperscript{Boot} and $h$\textsuperscript{BExp}, are based on the idea of training a heuristic on successively harder to solve states~\cite{arfaeebootstrap}. $h$\textsuperscript{Boot} and $h$\textsuperscript{BExp}, perform regressions from the goal following random walks, and solve using GBFS with the current $h$\textsuperscript{Boot} or $h$\textsuperscript{BExp} heuristic from a fully assigned state randomly sampled from the pre-image discovered by the regression. 
$h$\textsuperscript{Boot} uses the plan's length while $h$\textsuperscript{BExp} uses the number of states expanded by GBFS as the state training labels for the estimated goal distance.
$h$\textsuperscript{Boot} and $h$\textsuperscript{BExp} use the states and labels to iteratively train their NNs, and once the GBFS is solving 95\% of the states found from the regressions, the maximum length of the regression is doubled. 
Ferber et al.~\shortcite{ferber:2021:prl} also introduced the $h$\textsuperscript{AVI} heuristic which is trained using approximate value iteration. $h$\textsuperscript{AVI}, also discovers states through regressions but instead of solving for the state, $h$\textsuperscript{AVI} performs Bellman updates on a 2 step lookahead from the state, evaluating the leaf states of the lookahead with the current $h$\textsuperscript{AVI} heuristic or as 0 if they are goal states. Ferber et al.'s motivation for these Active Learning (AL) approaches verse SL ones is that SL is limited to instances small enough for training data generation.

Beyond learning per-instance NN defined heuristics, Shen et al.'s~\shortcite{shen2020learning} STRIPS Hypergraphs Networks (HGNs) learns per-domain and even domain independent NN defined heuristics. 
In this work we do not focus on domain independent NN defined heuristic functions, but we do compare against a STRIPS-HGN heuristic trained in a per-domain framework as described by Ferber et al.~\shortcite{ferber:2021:prl}. 

In this paper we also explore a version of the RSL algorithm that aims to maximise the structural diversity of the states in its selected sample. Width-based planning is one method that has been particularly effective in increasing the structural diversity of states considered during search~\cite{lipovetzky:12:width,lipovetzky:15:atari,frances2017purely,katz2017adapting}. Width-based planning prunes states considered as non-novel by a defined measure. The original width-based measure introduced by Lipovetzky et al.~\shortcite{lipovetzky:12:width} is that the novelty of a state is evaluated by the smallest tuple of atoms $t \subseteq s$, where $s$ is the first state that makes $t$ true in the search. Inspired by the success of the width-based novelty algorithms in classical planning, both in progression and regression \cite{lei2021width}, we create and experiment with a new novelty based regression algorithm.

\section{Regression Based Supervised Learning}
Given a planning problem $\Pi = \langle F,O,I,G\rangle$, RSL produces a training set ${\cal{D}} = \{(s_1, h_1), \dots (s_N, h_N)\}$ which is a set of states $s \in S$ paired with goal distance estimates $h$. To produce ${\cal{D}}$ RSL performs $N_r$ rollouts, each starting at the goal $G$ and applying $L$ times the classical planning regression operator. Each rollout from $G$ is a sequence of actions $\pi^j = (a^j_i)_{i=0}^{L-1}$, for $j = 1, \ldots, N_r$. Each $\pi^j$ produces a sequence of pre-images $x^j_{0}, x^j_{1}, \ldots, x^j_{L}$, where $x^j_0=G$ and $x^j_i \subseteq F$. The sequence of pre-images denotes a sequence of sets of states $\mathcal{R}^j = X^j_{0}, X^j_{1}, \ldots, X^j_{L}$, where $X^j_i = \{ s \ | \ x^j_i \subseteq s, s \in S\}$. See Figure~\ref{fig:novelty_def} for an example of this mapping of a pre-image represented by a partial truth assignment into its corresponding set of fully assigned states. By the definition of the regression operator, $X^j_{i-1}$ corresponds with the \emph{pre-image}, conditioned on $a^j_{i-1}$, of $X^j_i$. Therefore any state within $X^j_{i}$ can be reached from $X^j_{i-1}$ by applying action $a^j_{i-1}$. It follows that any state within $X^j_{i}$ can reach $X^j_{0}$ in at most $i$ transitions. Using this observation, RSL labels each state $s$ within its training set of states, $T_s$, with
\begin{align}
d(s)=\mbox{min}(i \ | \ s \in X^j_i, j \in \{1, \ldots, N_r\})
\label{eq:rsl_label}
\end{align}
that is, the smallest goal distance estimate of any of the state sets visited by the regressions which $s$ is also a member of.

\begin{algorithm}
 \caption{Overview of the RSL Algorithm}
\label{alg:RSL}
\textbf{Input}: $\Pi$\\
\textbf{Parameter}: $P_r, L, N_r, N_t, \mu$\\
\textbf{Output}: $h$\textsuperscript{RSL}
\begin{algorithmic}[1] 
\STATE $\mathcal{R}\gets\textsc{regression}(\Pi, L, N_r,\mu$)
\STATE $T_s \gets \textsc{sample\_states}(\mathcal{R}, P_r, N_t)$
\STATE $\cal{D} \gets \textsc{label}(\mathcal{R}, T_s)$
\STATE $h$\textsuperscript{RSL}$\gets \textsc{supervised\_learning}(\cal{D})$
\end{algorithmic}
\end{algorithm}

Algorithm~\ref{alg:RSL}, 
provides an overview of RSL. The hyper-parameters of RSL are the length of each regression $L$, the number of regressions to perform $N_r$, the number of training states to use $N_t$, the percentage of training states that are randomly sampled from the entire state space $P_r$, and a function that maps state regression trajectories into novelty levels $\mu$. RSL has three distinct steps, 1: extracting sets of state sets, $\mathcal{R} = \bigcup_{j=1}^{N_r}\mathcal{R}^j$, through performing regressions from the goal set (Alg.~\ref{alg:RSL} line 1), 2: sampling training states $T_s$ and labeling them with goal distance estimates using \eqref{eq:rsl_label} (Alg.~\ref{alg:RSL} lines 2-3), and 3: training the NN defined heuristic function (Alg.~\ref{alg:RSL} line 4) using empirical risk minimisation~\eqref{eq:empirical_risk_min}. 

\begin{figure}[t!]
    \centering
    \includegraphics[width=0.47\textwidth]{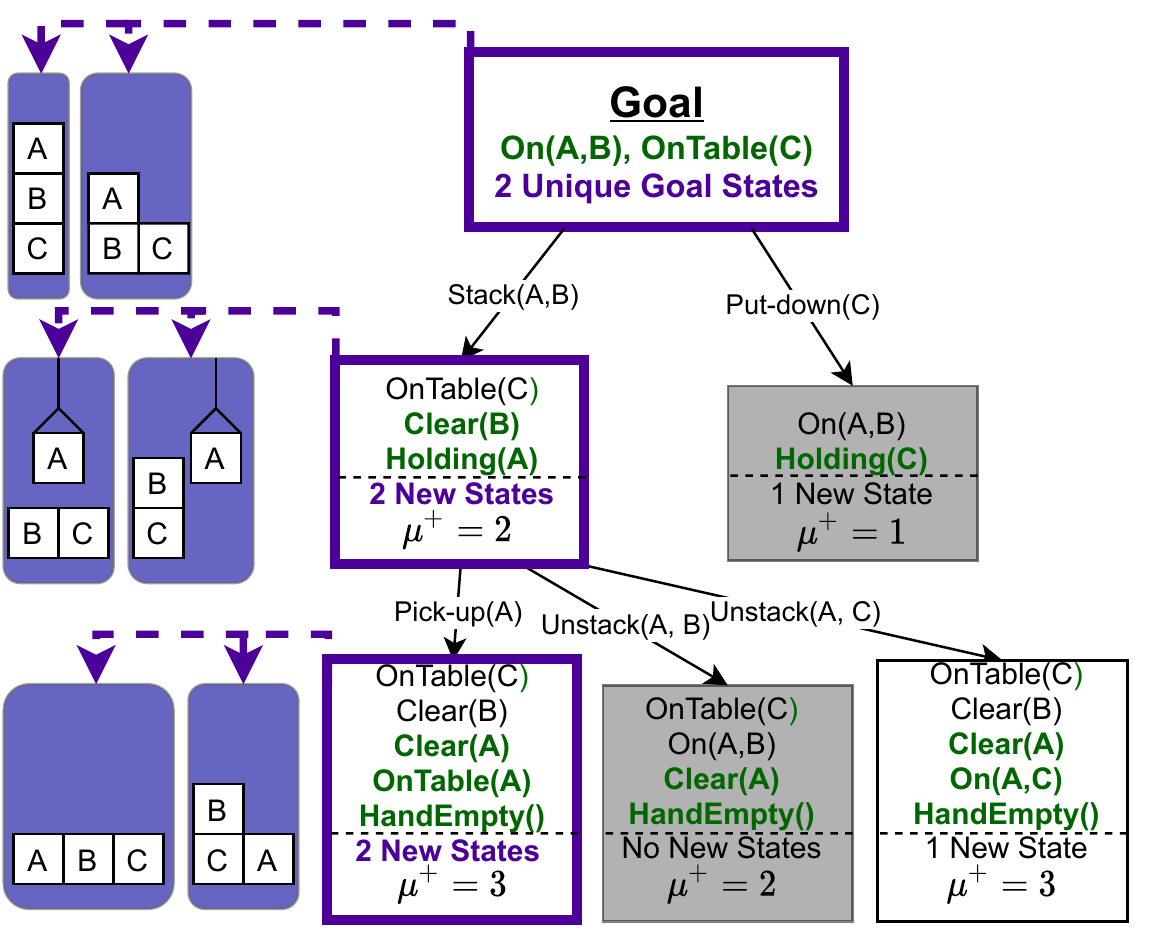}
    \caption{An example rollout performed by N-RSL with $L=2$ on a 3 block Blocksworld problem. Each square represents a per-image through a partially assigned state with its assigned atoms written within the square. The purple squares represent the example rollout trajectory selected by N-RSL and the possible Blocksworld states for each pre-image along the trajectory are shown on the left.  Green atoms are atoms assigned for the first time in the trajectory, and the $\mu^+$ value is calculated using Equation~\ref{nov:adds}.}
    \label{fig:novelty_def}
\end{figure}

\subsection{Extracting state sets through regression}
As previously explained RSL performs $N_r$ regressions to extract the set of state sets $\mathcal{R}$ over which the training set $\mathcal{D}$ is defined. 
At step $i$ of a rollout with a pre-image $x^j_{i}$ corresponding to the set of states $X^j_{i}$, as previously defined, the actions we consider valid for pre-imaging $X^j_{i}$ with are $a \in \nu(x^j_{i})$, defined as follows
\begin{align}
\begin{aligned}
\label{available_actions}
\nu(x^j_{i}) = \{ a \ &|\;a \in \textsc{reachable}(O, I), \\
&\;x^j_i \cap \mbox{e-Del}(a) = \emptyset, \\
&\;x^j_{i} \cap \mbox{Del}(a) = \emptyset, x^j_{i} \cap \mbox{Add(a)} \neq \emptyset \}
\end{aligned}
\end{align}

\noindent and e-Del$(a)$ is
\begin{align}
\begin{aligned}
\mbox{e-Del}(a) = \{ \ q \ |& \  q \in F \setminus Add(a), \\&\exists{p \in Pre(a)}: \textsc{mutex}(p,q) \}.
\label{eq:eDel}
\end{aligned}
\end{align}
$\textsc{reachable}(O, I)$ maps the operator set $O$ and initial state $I$ to the set of actions with reachable preconditions in the delete relaxation of $\Pi$~\cite{BONET20015} given the initial state of the progression state-transition model $I$. The $\textsc{mutex}(p, q)$ function in (\ref{eq:eDel}) maps the pair of atoms ($p$, $q$) to true if $p$ and $q$ are mutually exclusive, that is, it is impossible for $p$ and $q$ to both be true in any state $s \in S$ that can be reached from $I$. Note that the Ferber et al.~\shortcite{ferber:2021:prl} algorithms that use regression also filter the valid actions in the same way through using the mutex groups and applicable operations found by the Fast Downward (FD)~\cite{helmert2006fast} Translator\footnote{See the \textsc{is\_valid\_walk} parameter in line 848 in code/src/search/task\_utils/sampling\_technique.cc in Ferber et al.'s (2021) Supplementary Material available at \url{https://zenodo.org/record/5345958}  }.   

The baseline option for performing the rollout, and the method used by Ferber et al.~\shortcite{ferber:2021:prl}, is to randomly select actions $a$ for which applying the regression operator is valid. In addition to testing RSL using random action selection we instantiate a version of RSL we name Novelty guided Regression based Learning (N-RSL) that aims to increase the structural diversity of operators selected in its regression. 

\subsubsection{Preferring actions with novel preconditions}
For N-RSL at step $i$ of a rollout with a pre-image $x^j_{i}$, and a state trajectory of $\tau = G, x^j_{1}, \ldots, x^j_{i}$, N-RSL uses the novelty function $\mu$ to select the action $a_i$ to pre-image $X^j_{i}$.
\begin{align}
\begin{aligned}
a_i = \mbox{argmax}_a\{\mu(a, \tau) \ | \ a \in \nu(x^j_{i})\}
\label{eq:actionSel}
\end{aligned}
\end{align}
where ties are broken randomly and $\nu(x^j_{i})$ is a set of valid actions as described above. Note that for the plain RSL algorithm, which uses random action selection, $\mu=\mu(a, \tau) = 0$ for any $a \in O$ and any state trajectory $\tau$.

In order to increase the structural diversity of the states in sets $X^j_i$ in ${\cal R}^j$ we consider how the pre-images $x^j_i$ evolve as we repeatedly apply the regression operator. At each step $i$ of the regression we have a set of valid operators $\nu(x^j_{i})$ and a state set $X^j_{i}$ derived from the pre-image $x^j_i \subseteq F$, where $x^j_0=G$. We propose increasing the structural diversity of the sets extracted by counting the number of atoms an action will assign as true in a pre-image for the first time in the current trajectory. In order to apply this criterion, $\mu$ in \eqref{eq:actionSel} is replaced by $\mu^{+}$ below,
\begin{align}
\begin{aligned}
\mu^{+}(a, \tau) = |Pre(a) \setminus \bigcup_{x^j_i \in \tau} x^j_i|
\label{nov:adds}
\end{aligned}
\end{align}
This $\mu^{+}$ measure means N-RSL prefers actions with pre-conditions which contain atoms that are not specified in the goal $G$ and are not a member of the pre-condition set of any of the actions executed in the trajectory up until that point. Note that $\mu^{+}$ relies only on the current trajectory and is independent of any previously executed regression trajectories. 
Figure~\ref{fig:novelty_def} shows an example of a rollout performed on a simple Blocksworld problem with the one-step lookahead method described by \eqref{eq:actionSel} using the measure $\mu^{+}$ defined in \eqref{nov:adds}.

\subsection{Sampling and Labeling Training Data}
The training data for RSL is sampled from the sets of states $\mathcal{R}$.
States are sampled from each set $X_{i}^{j} \in {\cal R}$, and sampled states that contain mutex atom pairs $(p,q)$ are modified by removing either the $p$ or $q$ atom from the state. Note that if a sampled state from the set $X_{i}^{j}$ has a mutex pair $(p, q)$ and $p \in x_{i}^{j}$, $q$ will be removed from the state, that is, an atom that is a member of the partial state $x_{i}^{j}$ that a state is sampled from will never be removed from the state. As previously described, the labelled heuristic value for a sampled state is given by $d(s)$, which returns the distance of the closest state set in $\mathcal{R}$ to the goal according to the regression.

As we report later, some domains benefited from adding randomly sampled states from $S$. The random states have mutexes enforced using the same method as above, and are also labeled with $d(s)$. In the case that the state is not a member of any of the sets of states in $\mathcal{R}$, we define $d(s)$ to equal $L+1$.

\begin{userdefined}
Given a problem $\Pi=\langle F,O,I,G\rangle$, N-RSL obtains a training set $\cal{D}$ in $\mathcal{O}(|F|(|O|N_rL + N_tLN_r))$ time and $\mathcal{O}(|F|(N_rL + N_t))$ space.
\end{userdefined}
\begin{proof}[Proof sketch]
RSL performs $N_r$ rollouts and by definition the number of state transitions in each $\pi^j$ is $L$. For each state transition the one step lookahead (Equation~\ref{eq:actionSel}) considers a maximum of $|O|$ actions and $|O|$ pre-images.  
Therefore the total number of pre-images considered by RSL is $L \times |O| \times N_r$. Note that both $L$ and $N_r$ are set as constant hyper-parameters and for each pre-image a maximum of $|F|$ atoms need to be considered by the novelty definition $\mu^{+}$, therefore the $\textsc{regression}$ algorithm has a time complexity of $\mathcal{O}(|F||O|N_rL)$. Also for each pre-image a maximum of $|F|$ atoms need to be assigned, hence $\textsc{regression}$ runs in $\mathcal{O}(|F|N_rL)$ space. 
Sampling states requires the assignment of at most $|F|\times N_t$ atoms and checking a state's membership requires at most $|F|\times N_t \times L \times N_r$ comparisons meaning both the $\textsc{sample\_states}$ and $\textsc{label}$ algorithms run in $\mathcal{O}(|F||O|N_rL)$ time and $\mathcal{O}(|F|N_t)$ space.
\end{proof}

Given the training set, $\cal{D}$, any suitable off-the-shelf SL algorithm can be used to obtain a heuristic estimator through optimising the SL objective function (\ref{eq:empirical_risk_min}). We discuss 
the specifics of 
one such algorithm in our Experimental Study.

\section{Experimental Study}
The experimental study of RSL has two objectives: (1) evaluate the effect of the different mechanisms within the RSL algorithm, and (2) provide a direct comparison to existing NN and model-based heuristic functions.
\subsection{Methodology}
Our benchmark set of domains, instances and initial states is the same as Ferber et al.~\shortcite{ferber:2021:prl}. As we are learning per instance heuristics, a unique heuristic is trained for each problem instance and then evaluated over a set of 50 different initial states. The 50 initial states were produced by Ferber et al.~\shortcite{ferber2020neural,ferber:2021:prl}, through performing 50 $200$-step random-walks from the original initial state of the instance. The benchmark instances have also been separated by Ferber et al.~\shortcite{ferber2020neural,ferber:2021:prl} into ``Moderate Tasks", which are solved by GBFS guided by $h$\textsuperscript{FF} in less than 900 seconds but more than 1 second, and ``Hard Tasks" which are not solved within 900 seconds.

We evaluate $h$\textsuperscript{RSL} heuristic using the same method as Ferber et al.~\shortcite{ferber:2021:prl}. Each heuristic is evaluated over 50 different initial states guiding GBFS implemented in FD~\cite{helmert2006fast}. It is known that a lack of accuracy in heuristics can lead to blow ups in the size of search trees~\cite{helmert2008good}, therefore we can test the relative accuracy of the heuristics by using them in the GBFS and comparing their coverage. The coverage of a planner is defined as the percent of initial states for which a solution path is found within the given planning budget. Ferber et al.~\shortcite{ferber:2021:prl} report observing that in general the coverage superiority between the different NN heuristics tested did not vary over time. That is, the planning time used and the relative coverage superiority between the algorithms were not correlated. Given this observation and constraints on our available compute resources we reduce the overall 10 hour planning time-limit that Ferber et al. used by 99\% down to just 6 minutes, and compare with the NN defined heuristic baseline algorithms results as provided by Ferber et al. (2021) which use the 10 hour planning time-limit. 

As our results are run on different hardware and Ferber et al.'s algorithms use the Keras~\cite{chollet2015keras} and the TensorFlow~\cite{tensorflow2015} libraries to train and evaluate their NN while we use PyTorch~\cite{NEURIPS2019_9015}, we preformed a comparison of the evaluations per second used for commonly solved instances between Ferber et al.'s~\shortcite{ferber:2021:prl} methods and $h$\textsuperscript{N-RSL}. Ferber et al.'s algorithms use the same NN architecture meaning that for the same instance, and assuming the relationship between evaluations and planning time is linear, the evaluations per second should be similar. The data logs provided by Ferber et al. only provide the number of evaluations completed by FD if the problem is solved, therefore we can only compare the ratios of commonly solved instances.
Depending on the domain the ratios averaged between 0.3 to 16.1, meaning $h$\textsuperscript{RSL}'s evaluations per second were sometimes less but could also be up to 16.1 times higher. A full table of the minimum, maximum and average ratios found for each domain and algorithm can be found in the Supplementary Material.  
We also benchmark on our hardware an implementation of the SING~\cite{yu2020learning} algorithm which Ferber et al.~\shortcite{ferber:2021:prl} did not test and has not previously been applied to this benchmark set. Details of our implementation of the SING algorithm are provided in the Supplementary Material. Finally, we ran the baselines that do not require offline training, $h$\textsuperscript{FF} and LAMA, on our hardware with the 6 minute time budget.

We define each $h$\textsuperscript{RSL} heuristic through the same NN architecture used by Ferber et al.~\shortcite{ferber:2021:prl}. The NN is a \emph{residual network}~\cite{he2016deep} made up of two dense 250 neuron layers followed by a single residual block with two dense 250 neuron layers followed by a single neuron output. Each neuron in the NN uses the rectified linear activation function. The inputs of the NN are full states of $S(\Pi)$ represented by a Boolean vector $\{0,1\}^{\vert F \vert}$.
As Ferber et al. do, we set the $loss$ function of the SL optimization problem~\eqref{eq:empirical_risk_min} to be
the \emph{mean squared error} (MSE).
We use the Adam~\cite{kingma2014adam} stochastic optimization algorithm to find
locally optimal solutions of~\eqref{eq:empirical_risk_min}, using the following
hyper-parameters: learning rate is $10^{-4}$, batch size is $64$, and maximum
number of epochs is set to $1,000$.
Additionally, we use the \emph{early stopping} heuristic~\cite{duvenaud:early_stopping} setting
the \emph{patience} parameter to $2$. We split each instance data set into training and
validations sets, taking the former $80$\% and the latter $20$\% of available samples.
\begin{figure}[b!]
	\centering
	\vspace{-0.1in}
	\includegraphics[width=0.46\textwidth]{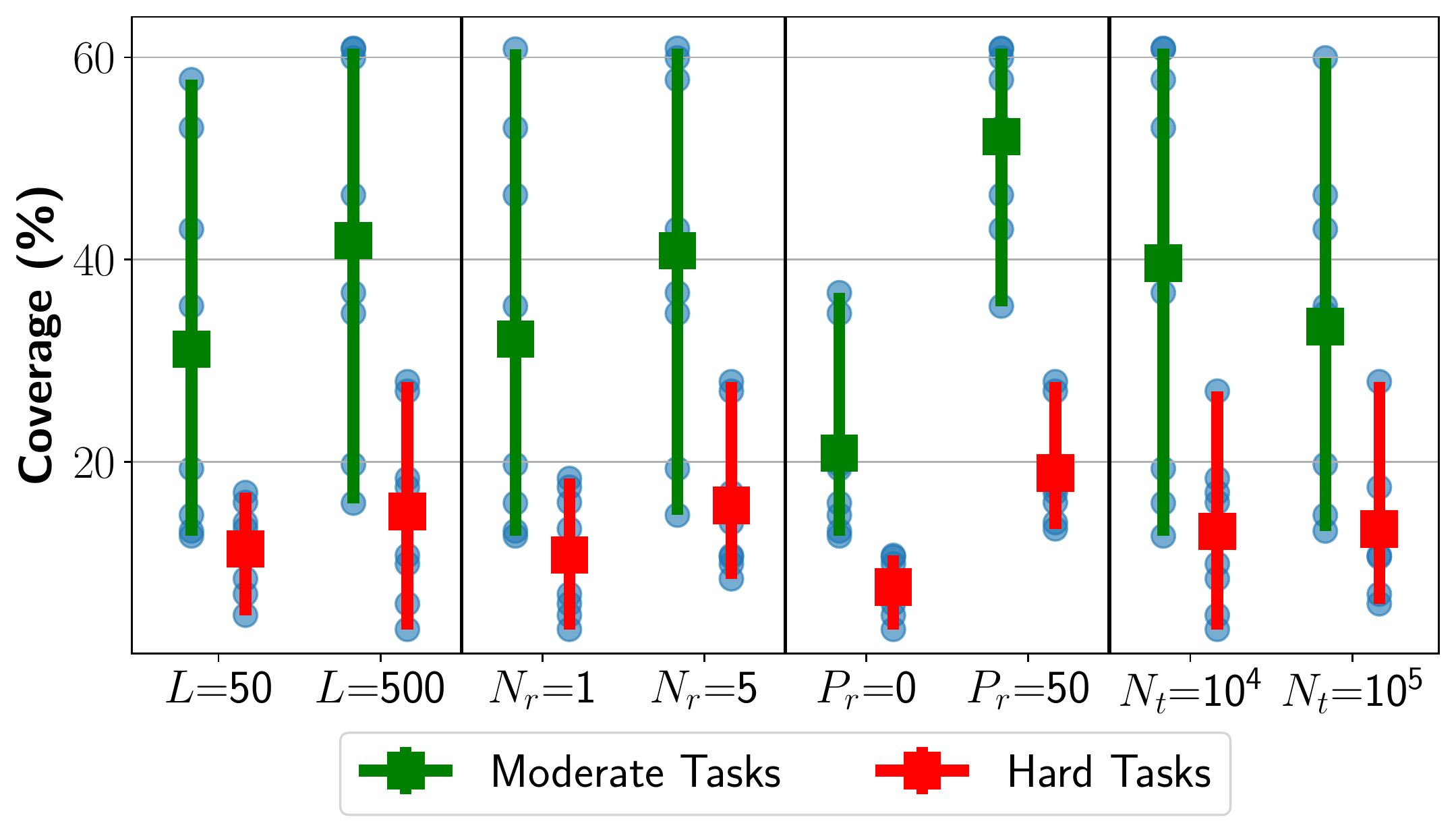}
	\caption{A summary of the average coverage over the ``Moderate" and ``Hard" tasks for 16 different configurations of RSL. The 4 boxes each show a different segmentation of the 16 configurations according to the value of a single hyper-parameter. Each configuration's coverage is marked with a blue dot and the lines show the range of coverage over the configurations with the mean marked with a square.}
	\vspace{-0.2in}
	\label{fig:grid}
\end{figure}
\begin{table*}[h!]
    \centering
    \small
    \setlength\tabcolsep{4pt} 
    \begin{tabular}{lcc|cc|ccccc|cc}
                                   & \multicolumn{2}{c|}{\textbf{Ave. over Trials}}        & \multicolumn{7}{c|}{\textbf{Methods using Validation}}  & & \\
                                  \textbf{Planning Budget} & \multicolumn{2}{c|}{6 minutes}      & \multicolumn{2}{c|}{6 minutes}   & \multicolumn{5}{c|}{600 minutes\textsuperscript{*}} & \multicolumn{2}{c}{6 minutes} \\   
                                   & $h$\textsuperscript{RSL} & $h$\textsuperscript{N-RSL} & $h$\textsuperscript{RSL} & $h$\textsuperscript{N-RSL}            & $h$\textsuperscript{Boot}          & $h$\textsuperscript{BExp} & $h$\textsuperscript{AVI}  & $h$\textsuperscript{TSL}            & $h$\textsuperscript{HGN}  & $h$\textsuperscript{FF}    &         LAMA    \\ \hline
\multicolumn{10}{c}{Moderate Tasks}                                                                                                                                                                                                                                        \\ \hline
blocks                             & 80.0                & 91.5                & 99.2                & \textbf{100}      & 18.0         & 0.0           & 0.0           & 80.4          & \textbf{100}       & 100               & 100             \\
depot                              & 48.7                & 58.8                & 77.0                & 89.0       & 60.3         & 32.7          & 54.7           & \textbf{90.3} & 0.0      & 93.7                & 99.3                                        \\
grid                               & 71.0                & 60.3                & \textbf{100}      & 97.0                & \textbf{100}         & \textbf{100}         & 51.0          & 93.0          & 0.0            & 90.0                & 100                                  \\
npuzzle                            & 15.3                & 18.9                & 25.2                & \textbf{46.8}       & 28.0          & 0.0          & 1.0           & 0.0           & 0.3          & 92.8                & 97.8                                \\
pipesworld                         & 70.7                & 69.6                & 85.8                & 82.6       & 57.8         & 68.4          & 50.2          & \textbf{92.2} & 7.6                  & 63.0                & 98.6                  \\
rovers                             & 12.6                & 12.5                & 15.0                & 15.8                & \textbf{48.2}        & 21.8 & 45.0          & 26.0          & 14.0               & 79.5                & 100                \\
scanalyzer                         & 87.4                & 94.1                & \textbf{100}      & \textbf{100}      & 33.3         & 70.7          & 67.3          & 82.7          & 11.0           & 91.7                & 100                              \\
storage                            & 16.1                & 16.4                & 17.0                & 18.0                & \textbf{89.0}         & 57.5 & 69.5          & 24.5          & 0.0               & 27.5                & 34.0                   \\
transport                          & 74.8                & 70.8                & \textbf{100}      & 93.8                & \textbf{100}        & \textbf{100}         & 87.5          & 99.2          & 94.7          & 100               & 100                                  \\
visitall                           & 93.6                & 95.4                & 99.7                & 98.0                & 55.3         & 0.0          & 0.0           & 0.0           & \textbf{100}                & 89.0                & 100           \\ \hline
\textbf{Average}                   & 57.0       & 58.8                & 71.9                & \textbf{74.1}       & 59.0                & 45.1          & 42.6     &58.8     & 32.8         & 82.7                & 93.0                                    \\ \hline
\multirow{2}{*}{\parbox{2.1cm}{\textbf{Average Train Time/Inst. (Hr)}}} & \multicolumn{1}{c}{\multirow{2}{*}{0.37}} & \multicolumn{1}{|c|}{\multirow{2}{*}{0.38}} & \multicolumn{1}{c}{\multirow{2}{*}{3.68}} & \multicolumn{1}{|c}{\multirow{2}{*}{3.75}} & \multicolumn{4}{|c}{\multirow{2}{*}{112\textsuperscript{*}}}                                          & \multicolumn{1}{|c|}{\multirow{2}{*}{13.1\textsuperscript{*}}} & \multicolumn{2}{c}{\multirow{2}{*}{-}   }      \\ & \multicolumn{1}{c}{} & \multicolumn{1}{|c|}{} & \multicolumn{1}{c}{} & \multicolumn{1}{|c}{} & \multicolumn{4}{|c}{}                                          & \multicolumn{1}{|c|}{} & \multicolumn{2}{c}{  } \ \\ \hline
\multicolumn{10}{c}{Hard Tasks}                                                                                                                                                                                                                                            \\ \hline
blocks                             & 18.5                & 45.3                & 36.4                & \textbf{68.0}       & 0.0          & 0.0           & 0.0           & 0.0           & 50.0       & 46.4                & 96.0                             \\
depot                              & 4.0                 & 3.3                 & 14.6                & 12.6               & 8.3           &4.3   & 12.9        & \textbf{35.4} & 0.0                & 9.4                 & 69.4               \\
grid                               & 21.9                & 32.1                & 67.2                & 69.0              & 87.8          & \textbf{95.0} & 70.5 & 60.2          & 0.0                & 31.0                & 100        \\
npuzzle                            & 0.0                 & 0.0                 & 0.0                 & 0.0                 & 0.0          & 0.0           & 0.0           & 0.0           & 0.0               & 5.8                 & 12.8                 \\
pipesworld                         & 27.4                & 25.0                & 33.8                & 33.1                 & 23.4          & 19.1    & 8.0       & \textbf{48.7} & 0.0        & 20.4                & 68.0              \\
rovers                             & 0.0                 & 0.1                 & 0.1                 & 0.1                 & 2.8           & 0.8     & \textbf{6.5}       & 1.5           & 0.3             & 8.3                 & 100                       \\
scanalyzer                         & 86.4                & 89.7                & \textbf{100}      & \textbf{100}           & 3.3           & 0.0    & 60.7       & 60.0          & 0.0                    & 83.3                & 100                \\
storage                            & 3.7                 & 2.5                 & 1.8                 & 0.2                 & \textbf{27.2} & 13.2     & 15.8      & 0.0           & 0.0                 & 9.2                 & 9.0         \\
transport                          & 0.0                 & 1.2                 & 0.0                 & \textbf{8.8}        & 0.0          & 0.0           & 2.4           & 0.0           & 0.0                 & 0.0                 & 59.6                       \\
visitall                           & 82.4                & 91.8                & \textbf{100}      & \textbf{100}               & 28.0          & 0.0   & 0.0        & 0.0           & \textbf{100}        & 40.0                & 100           \\ \hline
\textbf{Average}                   & 24.4                & 29.1                & 35.4                & \textbf{39.2}                & 18.1          & 13.3   & 17.7       & 20.6          & 15.0           & 25.4                & 71.5                        \\ \hline
\multirow{2}{*}{\parbox{2.1cm}{\textbf{Average Train Time/Inst. (Hr)}}} & \multicolumn{1}{c}{\multirow{2}{*}{0.69}} & \multicolumn{1}{|c|}{\multirow{2}{*}{0.69}} & \multicolumn{1}{c}{\multirow{2}{*}{6.88}} & \multicolumn{1}{|c}{\multirow{2}{*}{6.88}} & \multicolumn{4}{|c}{\multirow{2}{*}{112\textsuperscript{*}}  }                                        & \multicolumn{1}{|c|}{\multirow{2}{*}{10.4\textsuperscript{*}}}  & \multicolumn{2}{c}{\multirow{2}{*}{-}} \\  \multicolumn{1}{c}{} & \multicolumn{1}{c}{} & \multicolumn{1}{|c|}{} & \multicolumn{1}{c}{} & \multicolumn{1}{|c}{} & \multicolumn{4}{|c}{}      &\multicolumn{1}{|c|}{}      & \multicolumn{2}{c}{  }                                                
\end{tabular}
    \caption{Comparison of the coverage of $h$\textsuperscript{RSL} with other Neural Network defined heuristics functions introduced by Ferber et al. \shortcite{ferber:2021:prl} $h$\textsuperscript{Boot}, $h$\textsuperscript{BExp}, $h$\textsuperscript{AVI}, as well as the $h$\textsuperscript{TSL} introduced by Ferber et al. \shortcite{ferber2020neural}, and $h$\textsuperscript{HGN} from Shen et al. \shortcite{shen2020learning}. The table also shows the coverage of $h$\textsuperscript{ff}\cite{hoffmann2001ff} and LAMA \cite{richter2010lama} (run on same hardware as RSL). Note that $h$\textsuperscript{HGN} trains one heuristic per domain not instance with training budgets between 14.7 and 112 CPU hours. The bold numbers indicate the highest coverage among the Neural Network methods. \textsuperscript{*}Note that the baseline learning methods that use a 600 minute planning budget are the values as reported by Ferber et al.~\shortcite{ferber:2021:prl}. According to standard single thread CPU benchmarks, our vCPU can be 20\% faster. We provide a discussion of the impact of discrepancies in hardware and software in the Methodology Section. }
    \vspace{-0.1in}
    \label{tab:comp_baselines}
\end{table*}

Due to computational constraints we do not tune the architecture or any training parameters of the NN. Note that Ferber et al.'s TSL method trains 10 heuristics functions and selects the best performing heuristic on a set of validation states. We follow this method by running RSL 10 times with different random seeds on each instance to produce 10 heuristic functions. We then test each heuristic over a set of validation states, which are collected using the same method as Ferber et al.~\shortcite{ferber:2021:prl}, to select the best performing heuristic for each instance. We report both the average coverage of the 10 heuristics learnt and the best heuristic found via the validation method.
\begin{figure*}[t!]
    \centering
    \vspace{-0.07in}
    \begin{tabular}{ccc}
    \multicolumn{3}{c}{\textbf{Expansions}} \\
  \includegraphics[trim={0cm 0cm 0.2cm 1cm},clip, width=0.3\textwidth]{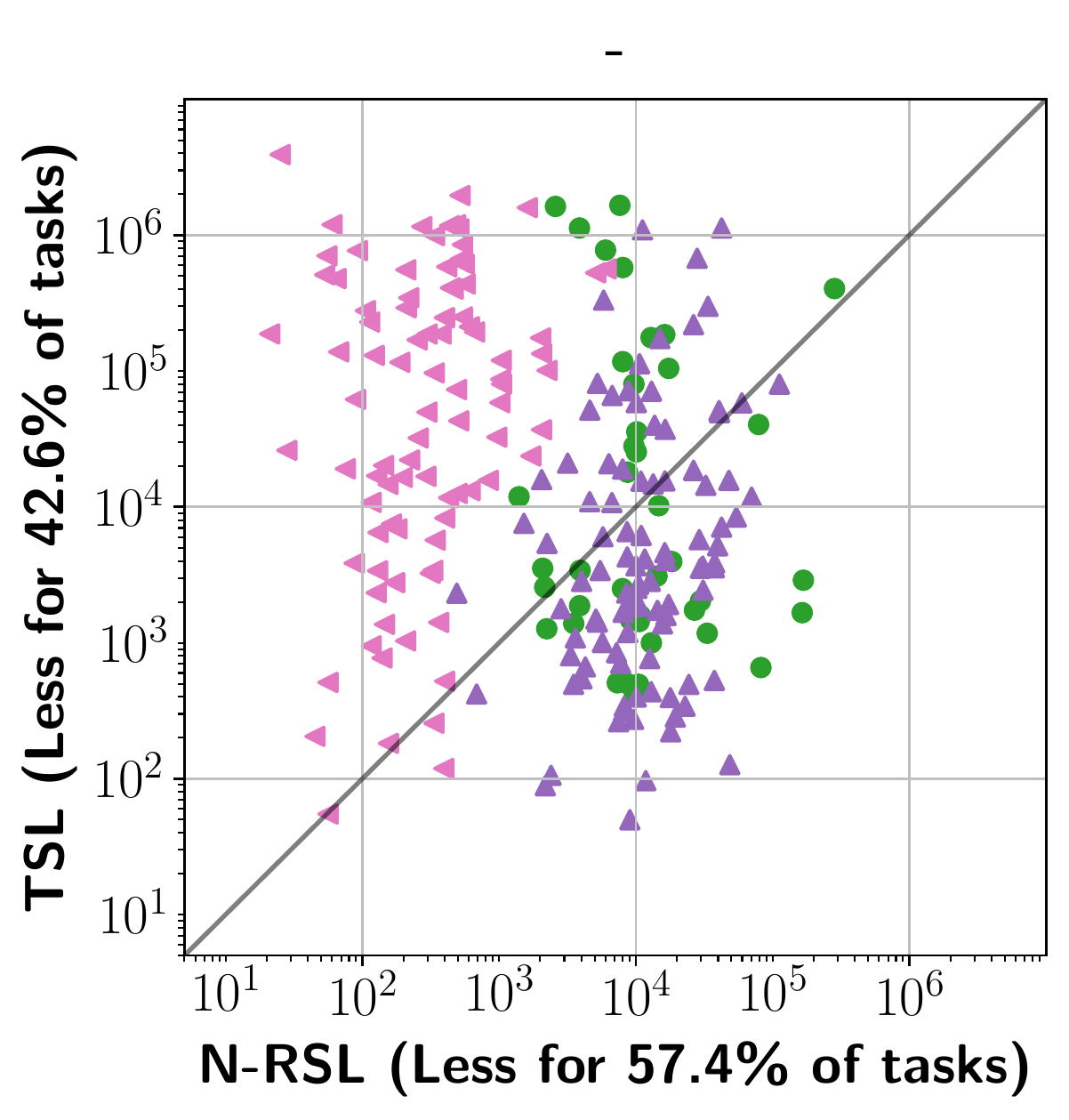} &   \includegraphics[trim={0cm 0cm 0.2cm 1cm},clip, width=0.3\textwidth]{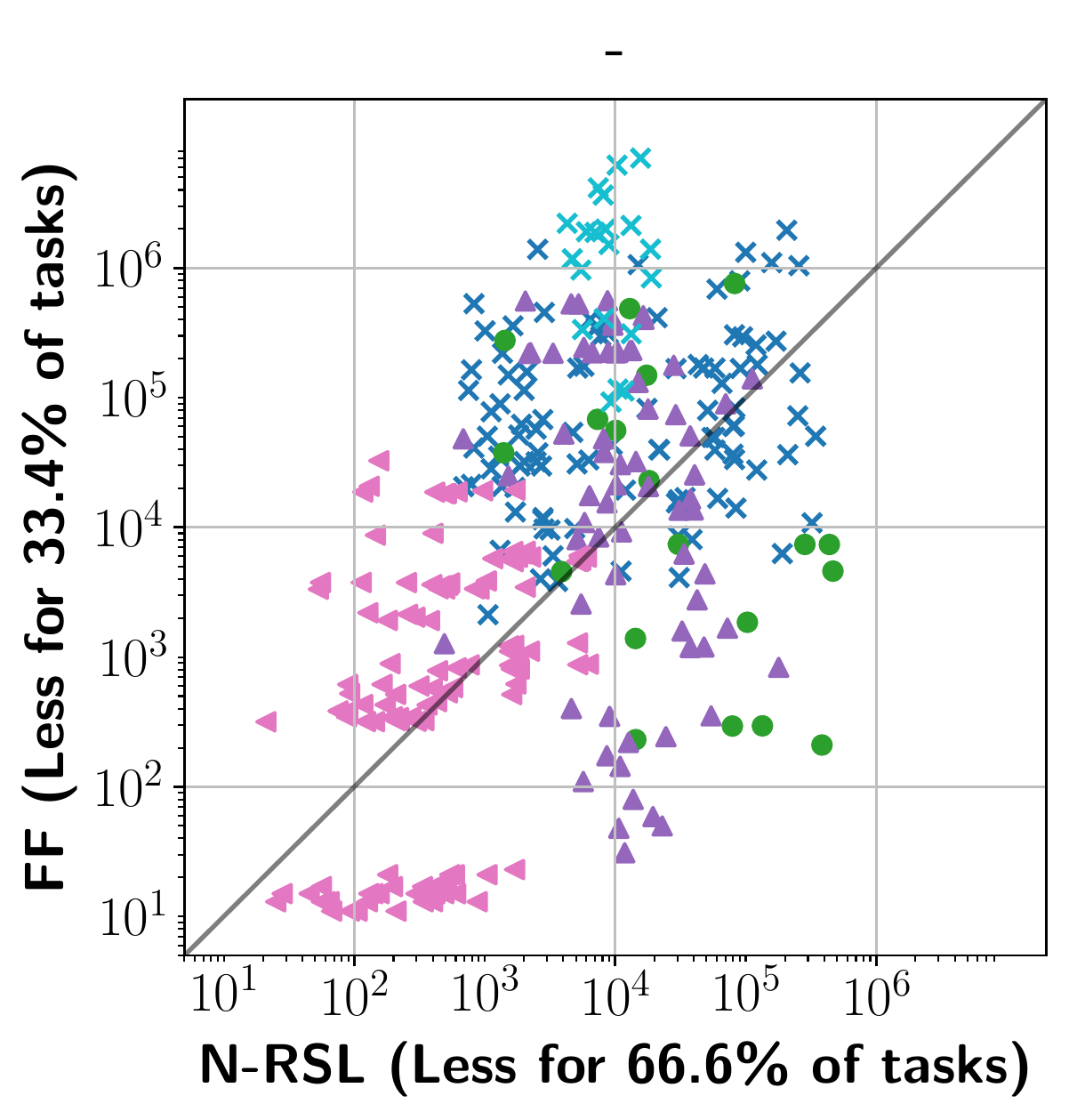} &   \includegraphics[trim={0cm 0cm 0.2cm 1cm},clip, width=0.3\textwidth]{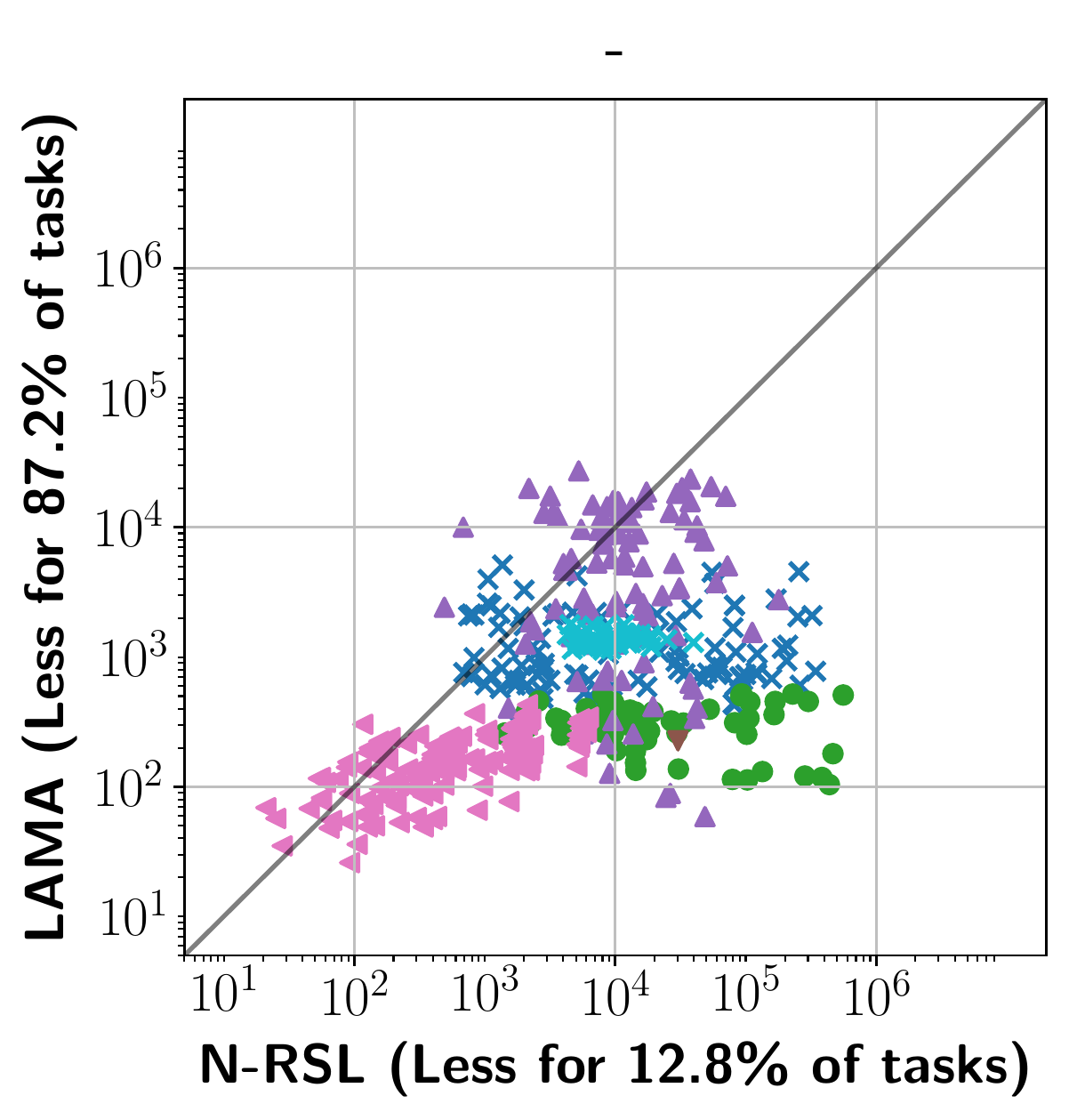} \\ \hline \\[-2ex]
  \multicolumn{3}{c}{\textbf{Plan Length}} \\
  \vspace{-0.01in}
  \includegraphics[trim={0cm 0cm 0.2cm 1cm},clip, width=0.3\textwidth]{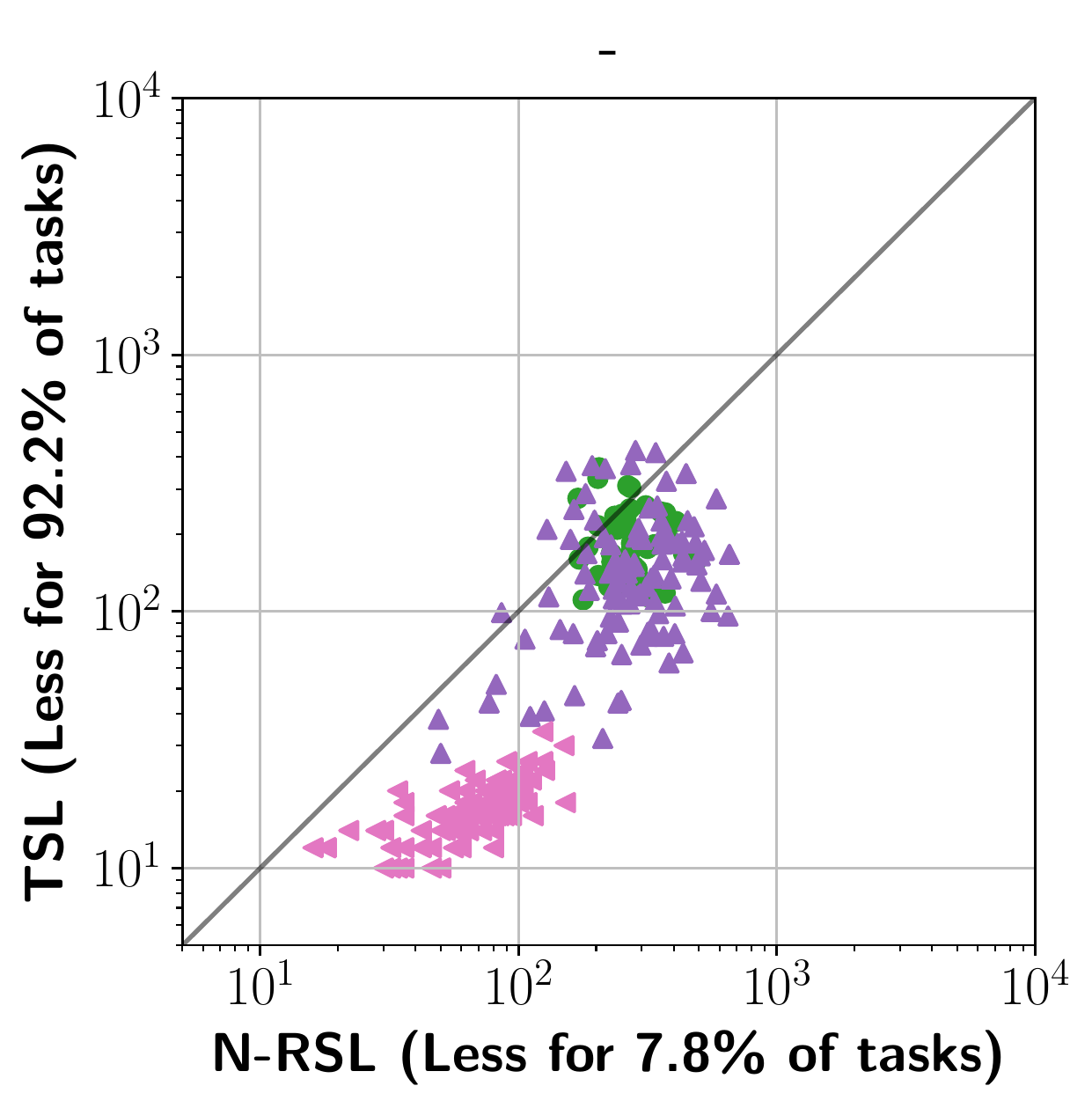} &   \includegraphics[trim={0cm 0cm 0.2cm 1cm},clip, width=0.3\textwidth]{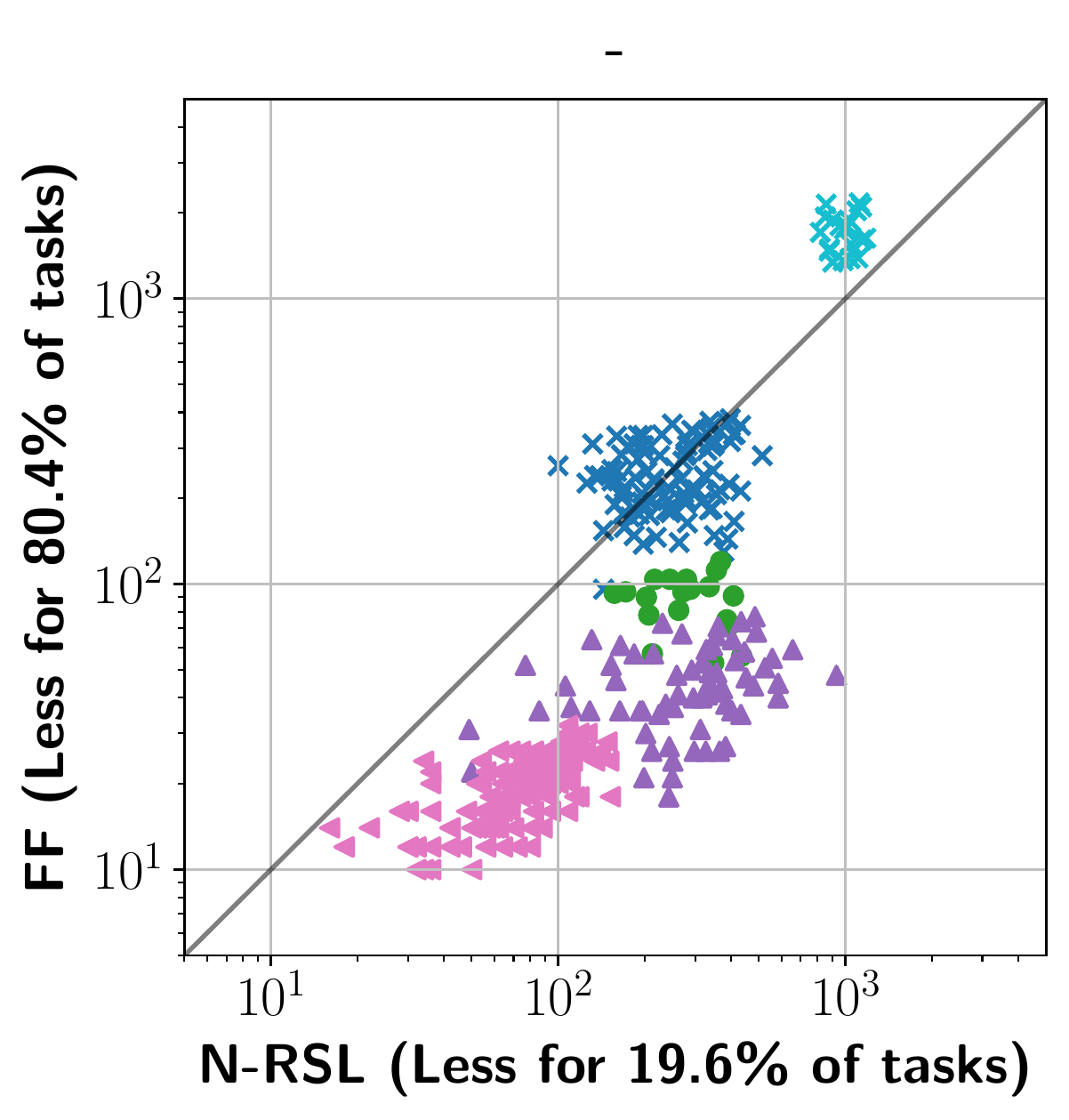}  &   \includegraphics[trim={0cm 0cm 0.2cm 1cm},clip, width=0.3\textwidth]{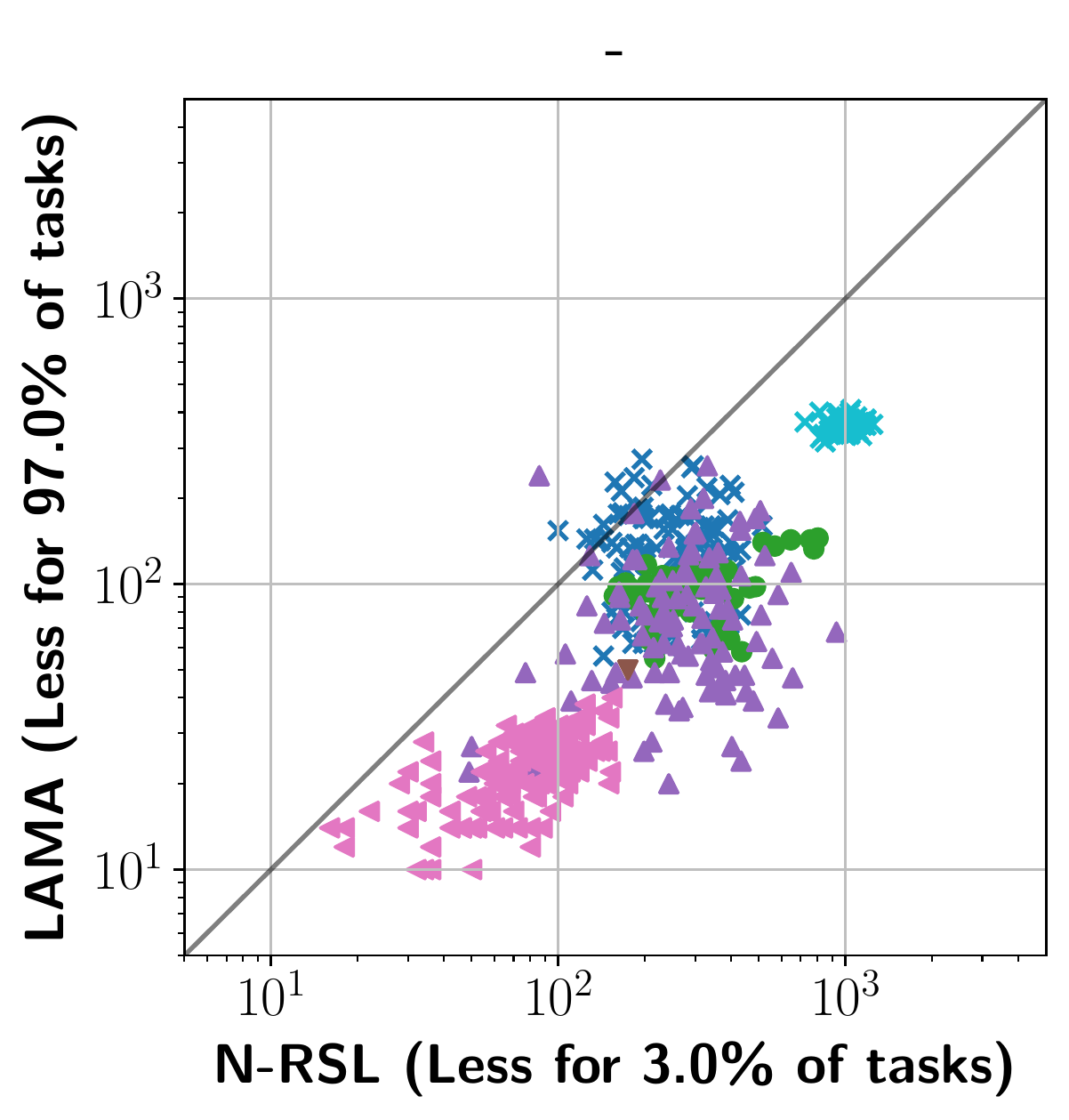} \\
  \multicolumn{3}{c}{\includegraphics[trim={0.5cm 6cm 0.5cm 5.5cm},clip, width=\textwidth]{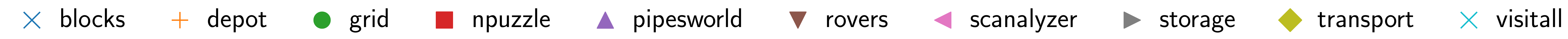}}
\end{tabular}
\vspace{-0.15in}
    \caption{Pairwise comparisons over commonly solved tasks between $h$\textsuperscript{N-RSL} using the validation method, and the best performing baseline NN based heuristic function $h$\textsuperscript{TSL} and model-based methods $h$\textsuperscript{FF} and LAMA over the ``Hard Tasks" benchmark set. The top row shows the scatter plot of the number of expansions used by each search algorithm while the bottom row shows the plan length found by each algorithm. The percentage value, shown in the axis titles, is the percentage of the commonly solved tasks that the relevant algorithm required fewer expansions or discovered a shorter plan than the algorithm it is being compared to. Note that tasks which are not solved by at least one of the algorithms in the pair are not included in these graphs. Graphs were generated using Downward Lab~\cite{seipp-et-al-zenodo2017}.}
    \label{fig:expansionsPlanQuality}
    \vspace{-0.2in}
\end{figure*}

\begin{figure}[t!]
    \centering

 \includegraphics[trim={0.5cm 0cm 1cm 1cm},clip, width=0.5\textwidth]{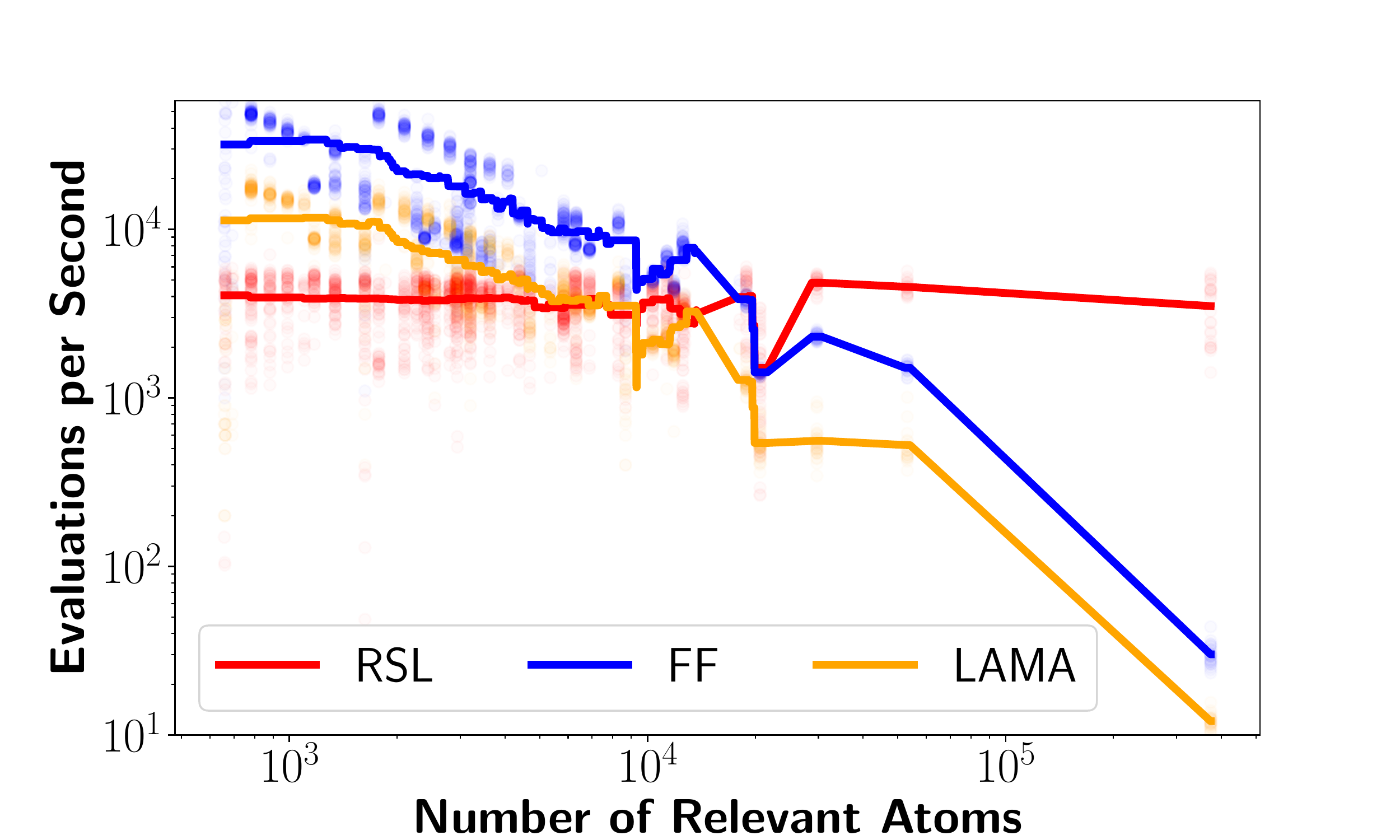}

    \caption{The number of heuristic evaluations per second versus the number of relevant atoms as determined by FD translator's reachability analysis for a set of commonly solved problem instances between $h$\textsuperscript{RSL}, $h$\textsuperscript{FF} and LAMA. Each dot represents the values for a single problem instance and the line represents the moving average with a window size of 2000 atoms.}
    \label{fig:expansionsPerSec}
\end{figure}

Training is run simultaneously for 80 instances over 40 Intel\textregistered Xeon\textregistered Gold 6138 CPU @ 2.00GHz processor cores with 720GB of shared RAM, limiting each training instance to run on a single vCPU (each CPU core has 2 vCPUs). For training the $h$\textsuperscript{RSL} heuristic, the Tarski framework~\cite{tarski} is used to ground each problem instance in order to get the $F$,$O$, and $G$ sets used for the regressions. In line with the $h$\textsuperscript{Boot}, $h$\textsuperscript{BExp}, $h$\textsuperscript{AVI} and $h$\textsuperscript{SL} heuristics, the FD translator is used for identifying a set of state mutexes to enforce, for both the state sampling and valid actions for the regression operator as described in the previous Section. For evaluation we use the same hardware, along with Downward lab~\cite{seipp-et-al-zenodo2017} to run 80 FD searches at once limiting each FD search to run on a single vCPU with a maximum of 3.8 GB of memory. The PyTorch library~\cite{NEURIPS2019_9015} is used for the training and evaluations of RSL's NN. The RSL code, result files and additional implementation details are included in the Supplementary Material.

\subsection{Results}

First we present the results from a grid search over the hyper-parameters of RSL. To maximise the number of hyper-parameter combinations we could test with our limited computational resources, we evaluated each trained heuristic from the grid search on 10 out of the 50 initial states for each of the 143 instances in the benchmark set. Figure~\ref{fig:grid} summarises the impact of each of the hyper-parameter values tested on the average coverage of $h$\textsuperscript{RSL} across the benchmark set. For the hyper-parameter grid search we tested 16 different settings over the different combinations of $N_t=$10,000 or 100,000, $P_r=$ 0 or 50, $N_r=$ 1 or 5, and $L=$ 50 or 500. The Supplementary Material includes a table detailing the different configurations tested along with their average coverage over each domain. Figure~\ref{fig:grid} shows that the value of $P_r$ has the biggest influence on the average coverage over the benchmark set of $h$\textsuperscript{RSL} used in GBFS. This observation shows the importance of including states in the training set that are not within any of the sets of states visited by the rollouts from the goal set. Figure~\ref{fig:grid} also suggests that RSL benefits from using more than one rollout for extracting training states and considering 500 applications of the regression operator instead of 50. Interestingly, Figure~\ref{fig:grid} shows that training with 10,000 or 100,000 states does not influence the performance in an obvious way. However, overall the best performing set of hyper-parameters over the ``Hard Tasks" which we use for all experiments from this point forth are $N_t=$100,000, $P_r=$50, $N_r$=5, and $L=$500.

Table~\ref{tab:comp_baselines} shows a comparison of existing methods with respect to RSL and N-RSL, using the best performing configuration from the hyper-parameter grid search. The SING algorithm is omitted from Table~\ref{tab:comp_baselines} due to poor performance but its results are included in the Supplementary Material. The first notable difference is the average training times used by each algorithm. $h$\textsuperscript{RSL} without validation uses less than 1\% of the training CPU time used by all the other per-instance NN defined heuristics functions for both the ``Moderate" and ``Hard" task sets. Even when using the validation method which requires 10 RSL executions the training time used is less than 7\% of the per-instance NN methods. The overall average coverage over the benchmark set shows that $h$\textsuperscript{RSL} and $h$\textsuperscript{N-RSL} outperform the other NN defined heuristic functions for the ``Hard Tasks" with and without using the validation method. $h$\textsuperscript{RSL} and $h$\textsuperscript{N-RSL} when using the validation method also outperform the other NN heuristics on the ``Moderate Tasks". The model-based method LAMA clearly dominates all other methods. The model-based heuristic $h$\textsuperscript{FF} also outperforms $h$\textsuperscript{RSL} and $h$\textsuperscript{N-RSL} on the ``Moderate Tasks", however on the ``Hard Tasks" $h$\textsuperscript{RSL} with validation and $h$\textsuperscript{N-RSL} with and without validation have better overall performance. While $h$\textsuperscript{N-RSL} has better average coverage than $h$\textsuperscript{RSL} the difference is small with a 1.8\% and 4.7\% average improvement over the ``Moderate" and ``Hard Tasks" respectively when no validation is used.

Figure~\ref{fig:expansionsPlanQuality} provides additional insights into the performance of $h$\textsuperscript{N-RSL} compared to the best performing baseline algorithms. For commonly solved tasks, $h$\textsuperscript{TSL} on average produces higher quality plans than $h$\textsuperscript{N-RSL}. This observation could be a result of $h$\textsuperscript{TSL} having higher quality labels used in training. $h$\textsuperscript{TSL} uses $h$\textsuperscript{FF} with GBFS to label sampled states in the train set, while $h$\textsuperscript{N-RSL} labels states with goal distance estimates derived from the state sets visited by N-RSL's regressions. The objective of $h$\textsuperscript{FF} with GBFS is to find the shortest path from a state to the goal state, while the objective of N-RSL's regression is to maximise a defined novelty measure (Equation~\ref{eq:actionSel}) with the aim of finding the shortest path that visits the most number of reachable atoms $a \in F$. This trade-off between exploration and exploitation shows through in the results where $h$\textsuperscript{N-RSL} has a higher coverage than $h$\textsuperscript{TSL} over a diverse set of initial states but produces lower quality plans in terms of plan length among the commonly solved tasks. This behaviour is also displayed when comparing directly to the $h$\textsuperscript{FF} heuristic, while LAMA dominates $h$\textsuperscript{N-RSL} in both coverage and plan quality.

Figure~\ref{fig:expansionsPerSec} shows that for smaller problem instances the model-based methods $h$\textsuperscript{FF}, and LAMA are able to compute a higher number of evaluations per second than $h$\textsuperscript{RSL}. However, $h$\textsuperscript{RSL}'s evaluations per second stays relatively constant while the model-based methods evaluations per second decreases as problem instances become larger.

\section{Discussion}
The heuristic values from the NN defined heuristic functions we report in this paper are calculated by taking the linear combination of the outputs of the last hidden layer of the NN. We note that the NN heuristics discussed in this paper and those of Ferber et al.~\shortcite{ferber2020neural,ferber:2021:prl} are an application of \emph{state aggregation}~\cite{bertsekas2018feature}, a well-known technique in Approximate Dynamic Programming. 
For each of the different problem instances we explored in this paper, the size of the last hidden layer was fixed at 250 neurons. It is possible that for harder problem instances a feature vector larger than 250 is required in order to produce more useful heuristics. Ferber et al.~\shortcite{ferber2020neural} have performed an investigation into effects of the NN architecture in terms of the activation functions used and the number of hidden layers, however they did not explore the number of neurons used within the hidden layers. One avenue of future work is to investigate the impact of scaling the number of neurons used within the hidden layers of the NN so the capacity~\cite{goodfellow2016deep} of the NN scales too with the size of the planning instance. 

However, one advantage we observed of the NN heuristic functions over the model-based heuristics in Figure ~\ref{fig:expansionsPerSec} is that the gap in number of evaluated states per second narrows until it eventually disappears, so the NN heuristic becomes significantly faster in comparison (and less informed too, probably) as the problem size increases. This is because only the input layer size changes according to $|F|$ for the NN architecture used in this work, while the rest of the neurons and edges remain constant. That is, as the input layer has $|F|$ neurons and the first hidden layer has 250 neurons there are $|F| \times 250$ edges joining the two layers resulting in the NN evaluations being in $O(|F|)$ time. Conversely, $h$\textsuperscript{FF} requires completing a search over the problem with the delete relaxation which has time-complexity that is polynomial over the number of actions and reachable atoms~\cite{hoffmann2001ff}. 

Future work should investigate NN heuristics like $h$\textsuperscript{RSL} over problems that are too large for $h$\textsuperscript{FF} or LAMA to produce enough evaluations per second to be useful. It is possible that a NN defined heuristic function will be able to produce a large enough number of evaluations per second to be useful over these large domains, even if the heuristic is less informed than the model-based heuristics. Last, we have shown that $h$\textsuperscript{RSL} and $h$\textsuperscript{N-RSL} have good performance while requiring significantly smaller resources than the NN-based baselines considered in the paper. We think that this encouraging result supports the assumption that RSL and follow-up methods can be used to obtain useful heuristic functions for large instances of challenging combinatorial optimization problems.
\clearpage

\bibliography{lookaheads}

\clearpage

\section{\huge{Supplementary Material}}
\vspace{1cm}
\section{Additional RSL Implementation Details}
For the NN training we use the Adam optimizer~\cite{kingma2014adam} with a learning rate of $10^{-4}$, initial decay rates of $\beta^1=0.9$ and $\beta^2=0.999$, and $\epsilon=10^{-8}$. Additionally we use a batch size of 64, a maximum of 1,000 training epochs, and early stopping with a patience of 2, using a split of training samples of 80-20 for training and validation, respectively.

\section{Comparison of Evaluations per Second of Baseline Results}
Table~\ref{tab:norm} shows the evaluations per second ratios between our $h$\textsuperscript{RSL} FastDownward runs and the baseline algorithms run by Ferber et al.~\shortcite{ferber:2021:prl}. 

\section{SING Implementation Details}
We implemented the configuration C4 of the SING algorithm~\cite{yu2020learning} as described by Yu et al. but using the same NN architecture, loss function and training hyper-parameters as used for RSL. $10^5$ samples were used for training, which were collected using a budget of 500 samples per DFS resulting in 200 total DFS rollouts.

\section{Additional Results}
The results for our implementation of SING are shown in Table~\ref{tab:sing}. The performance of SING is very poor compared to the results of the other baselines shown in Table 1 in the main paper. We also performed individual runs of instances using SING with the same NN architecture, loss function, number of training samples, and number of samples per DFS as d
escribed by Yu et al.~\shortcite{yu2020learning}, however the performance observed was still poor. Yu et al. mention that they perform manual tuning of hyper-parameters but do not provide the selected parameters used by the NN training algorithm, which could be the reason for the poor performance observed in our results.

Table~\ref{tab:grid_baselines} shows the results of each configuration tested in the RSL hyper-parameter grid search.

\begin{table}[h!]
    \centering
    \small
    \setlength\tabcolsep{4pt} 
    \begin{tabular}{lc|c}
\multicolumn{1}{c}{}                   & \textbf{Moderate Tasks} & \textbf{Hard Tasks} \\
blocks                                 & 17.6                    & 0.8                 \\
depot                                  & 0.0                     & 0.0                 \\
grid                                   & 0.0                     & 0.0                 \\
npuzzle                                & 0.0                     & 0.0                 \\
pipesworld                & 13.8                    & 0.1                 \\
rovers                                 & 6.2                     & 0.0                 \\
scanalyzer              & 16.7                    & 0.0                 \\
storage                                & 0.0                     & 1.0                 \\
transport               & 0.0                     & 0.0                 \\
visitall                 & 1.3                     & 0.0                 \\ \hline
\textbf{Average}                       & 5.6                     & 0.2                 \\ \hline
\multirow{2}{*}{\parbox{2.1cm}{\textbf{Average Train Time/Inst. (Hr)}}} & 1.6                     & 2.6  \\
& &
\end{tabular}
    \vspace{-0.07in}
    \caption{Coverage of the heuristic learnt by SING (as detailed in the ``SING Implementation Details" Section) in GBFS using a planning time limit of 6 minutes.}
    \label{tab:sing}
\end{table}
\begin{table*}[h!]
    \centering
    \small
    \setlength\tabcolsep{4pt} 
    \begin{tabular}{lccc|ccc|ccc|ccc}
                              & \multicolumn{3}{c|}{\textbf{$h$\textsuperscript{Boot}}}                                                                      & \multicolumn{3}{c|}{\textbf{$h$\textsuperscript{BExp}}}                                                                   & \multicolumn{3}{c|}{\textbf{$h$\textsuperscript{AVI}}}                                                                       & \multicolumn{3}{c}{\textbf{$h$\textsuperscript{TSL}}}                                                                       \\
\multicolumn{1}{c}{\textbf{}} & \multicolumn{1}{c}{\textbf{Ave}} & \multicolumn{1}{c}{\textbf{Min}} & \multicolumn{1}{c|}{\textbf{Max}} & \multicolumn{1}{c}{\textbf{Ave}} & \multicolumn{1}{c}{\textbf{Min}} & \multicolumn{1}{c|}{\textbf{Max}} & \multicolumn{1}{c}{\textbf{Ave}} & \multicolumn{1}{c}{\textbf{Min}} & \multicolumn{1}{c|}{\textbf{Max}} & \multicolumn{1}{c}{\textbf{Ave}} & \multicolumn{1}{c}{\textbf{Min}} & \multicolumn{1}{c}{\textbf{Max}} \\ \hline
blocks                        & 1.2                              & 0.7                              & 5.0                               & -                                & -                                & -                                 & -                                & -                                & -                                 & 2.3                              & 0.4                              & 6.8                              \\
depot                         & 1.2                              & 0.8                              & 1.4                               & 1.2                              & 1.0                              & 1.5                               & 1.2                              & 0.9                              & 2.5                               & 3.9                              & 0.7                              & 9.9                              \\
grid                          & 3.1                              & 0.6                              & 43.8                              & 1.3                              & 0.5                              & 12.7                              & 15.8                             & 0.7                              & 79.8                              & 13.2                             & 4.2                              & 80.5                             \\
npuzzle                       & 1.1                              & 0.8                              & 2.3                               & -                                & -                                & -                                 & 1.5                              & 1.0                              & 2.5                               & -                                & -                                & -                                \\
pipesworld                    & 1.4                              & 0.3                              & 8.0                               & 1.4                              & 0.7                              & 10.8                              & 1.4                              & 0.3                              & 19.7                              & 7.8                              & 0.8                              & 82.6                             \\
rovers                        & 14.6                             & 0.9                              & 70.1                              & 15.3                             & 0.9                              & 72.3                              & 16.1                             & 0.9                              & 85.6                              & 13.7                             & 0.3                              & 43.4                             \\
scanalyzer                    & 1.1                              & 0.1                              & 23.2                              & 0.3                              & 0.03                             & 0.9                               & 1.2                              & 0.1                              & 9.5                               & 8.6                              & 0.2                              & 45.8                             \\
storage                       & 1.5                              & 0.8                              & 4.5                               & 1.3                              & 0.8                              & 2.0                               & 1.2                              & 0.8                              & 1.6                               & 3.3                              & 1.0                              & 6.4                              \\
transport                     & 1.3                              & 0.8                              & 3.8                               & 1.1                              & 0.9                              & 1.7                               & 1.0                              & 0.7                              & 2.9                               & 5.0                              & 0.6                              & 11.2                             \\
visitall                      & 0.9                              & 0.4                              & 1.5                               & -                                & -                                & -                                 & -                                & -                                & -                                 & -                                & -                                & -                              
\end{tabular}
    \vspace{-0.07in}
    \caption{Ratios of evaluations per second of our $h$\textsuperscript{RSL} verse the baseline methods run by Ferber et al.~\shortcite{ferber:2021:prl}} for commonly solved instances.
    \label{tab:norm}
\end{table*}

\begin{table*}[h!]
    \centering
    \small
    \setlength\tabcolsep{4pt} 
    \begin{tabular}{lcccccccccccccccc}
\textbf{$N_t$} & \multicolumn{8}{c|}{\textbf{10,000}}                                                                                                                                                                                                                                                             & \multicolumn{8}{c}{\textbf{100,000}}                                                                                                                                                                                                                                            \\ \hline
\textbf{$P_r$} & \multicolumn{4}{c|}{\textbf{0}}                                                                                                               & \multicolumn{4}{c|}{\textbf{50}}                                                                                                                 & \multicolumn{4}{c|}{\textbf{0}}                                                                                                               & \multicolumn{4}{c}{\textbf{50}}                                                                                                 \\ \hline
\textbf{$N_r$}      & \multicolumn{2}{c|}{\textbf{1}}                                       & \multicolumn{2}{c|}{\textbf{5}}                                       & \multicolumn{2}{c|}{\textbf{1}}                                        & \multicolumn{2}{c|}{\textbf{5}}                                         & \multicolumn{2}{c|}{\textbf{1}}                                       & \multicolumn{2}{c|}{\textbf{5}}                                       & \multicolumn{2}{c|}{\textbf{1}}                                           & \multicolumn{2}{c}{\textbf{5}}                      \\ \hline
\textbf{$L$}             & \multicolumn{1}{c|}{\textbf{50}}  & \multicolumn{1}{c|}{\textbf{500}} & \multicolumn{1}{c|}{\textbf{50}}  & \multicolumn{1}{c|}{\textbf{500}} & \multicolumn{1}{c|}{\textbf{50}}  & \multicolumn{1}{c|}{\textbf{500}}  & \multicolumn{1}{c|}{\textbf{50}}   & \multicolumn{1}{c|}{\textbf{500}}  & \multicolumn{1}{c|}{\textbf{50}}  & \multicolumn{1}{c|}{\textbf{500}} & \multicolumn{1}{c|}{\textbf{50}}  & \multicolumn{1}{c|}{\textbf{500}} & \multicolumn{1}{c|}{\textbf{50}}    & \multicolumn{1}{c|}{\textbf{500}}   & \multicolumn{1}{c|}{\textbf{50}}    & \textbf{500}   \\ \hline
\multicolumn{17}{c}{\textbf{Moderate Tasks}}                                                                                                                                                                                                                                                                                                                                                                                                                                                                                                                                                                                 \\ \hline
blocks                                 & \multicolumn{1}{c|}{0.0}          & \multicolumn{1}{c|}{58.4}         & \multicolumn{1}{c|}{20.8}         & \multicolumn{1}{c|}{99.2}         & \multicolumn{1}{c|}{99.2}         & \multicolumn{1}{c|}{93.6}          & \multicolumn{1}{c|}{92.4}          & \multicolumn{1}{c|}{\textbf{100.0}}  & \multicolumn{1}{c|}{0.0}          & \multicolumn{1}{c|}{60.0}         & \multicolumn{1}{c|}{0.0}          & \multicolumn{1}{c|}{99.6}         & \multicolumn{1}{c|}{94.8}           & \multicolumn{1}{c|}{99.6}           & \multicolumn{1}{c|}{100.0}            & 98.8           \\
depot                                  & \multicolumn{1}{c|}{0.0}          & \multicolumn{1}{c|}{0.3}          & \multicolumn{1}{c|}{20.3}         & \multicolumn{1}{c|}{12.3}         & \multicolumn{1}{c|}{59.7}         & \multicolumn{1}{c|}{61.3}          & \multicolumn{1}{c|}{\textbf{79.3}} & \multicolumn{1}{c|}{\textbf{79.3}} & \multicolumn{1}{c|}{2.3}          & \multicolumn{1}{c|}{0.0}            & \multicolumn{1}{c|}{0.7}          & \multicolumn{1}{c|}{1.7}          & \multicolumn{1}{c|}{34.3}           & \multicolumn{1}{c|}{43.7}           & \multicolumn{1}{c|}{69.3}           & 51             \\
grid                                   & \multicolumn{1}{c|}{0.0}          & \multicolumn{1}{c|}{0.0}          & \multicolumn{1}{c|}{0.0}          & \multicolumn{1}{c|}{1.0}          & \multicolumn{1}{c|}{21.0}         & \multicolumn{1}{c|}{\textbf{84.0}} & \multicolumn{1}{c|}{50.0}          & \multicolumn{1}{c|}{40.0}          & \multicolumn{1}{c|}{0.0}          & \multicolumn{1}{c|}{0.0}          & \multicolumn{1}{c|}{0.0}          & \multicolumn{1}{c|}{5.0}          & \multicolumn{1}{c|}{16.0}           & \multicolumn{1}{c|}{9.0}            & \multicolumn{1}{c|}{1.0}            & 48             \\
npuzzle                                & \multicolumn{1}{c|}{0.0}          & \multicolumn{1}{c|}{0.0}          & \multicolumn{1}{c|}{0.0}          & \multicolumn{1}{c|}{0.0}          & \multicolumn{1}{c|}{0.2}          & \multicolumn{1}{c|}{28.8}          & \multicolumn{1}{c|}{1.2}           & \multicolumn{1}{c|}{28.8}          & \multicolumn{1}{c|}{0.0}          & \multicolumn{1}{c|}{0.0}          & \multicolumn{1}{c|}{0.0}          & \multicolumn{1}{c|}{0.0}          & \multicolumn{1}{c|}{0.0}            & \multicolumn{1}{c|}{17.2}           & \multicolumn{1}{c|}{1.2}            & \textbf{34.8}  \\
pipesworld-nt                          & \multicolumn{1}{c|}{11.8}         & \multicolumn{1}{c|}{15.8}         & \multicolumn{1}{c|}{17.6}         & \multicolumn{1}{c|}{23.6}         & \multicolumn{1}{c|}{28.6}         & \multicolumn{1}{c|}{41.4}          & \multicolumn{1}{c|}{27.2}          & \multicolumn{1}{c|}{\textbf{48}}   & \multicolumn{1}{c|}{21.6}         & \multicolumn{1}{c|}{26}           & \multicolumn{1}{c|}{31.4}         & \multicolumn{1}{c|}{23.4}         & \multicolumn{1}{c|}{27}             & \multicolumn{1}{c|}{41}             & \multicolumn{1}{c|}{30}             & 39             \\
rovers                                 & \multicolumn{1}{c|}{6.5}          & \multicolumn{1}{c|}{6.2}          & \multicolumn{1}{c|}{6.5}          & \multicolumn{1}{c|}{7.2}          & \multicolumn{1}{c|}{11.5}         & \multicolumn{1}{c|}{3}             & \multicolumn{1}{c|}{12.5}          & \multicolumn{1}{c|}{\textbf{12.8}} & \multicolumn{1}{c|}{6.5}          & \multicolumn{1}{c|}{6.2}          & \multicolumn{1}{c|}{4.5}          & \multicolumn{1}{c|}{5.2}          & \multicolumn{1}{c|}{7.5}            & \multicolumn{1}{c|}{1.8}            & \multicolumn{1}{c|}{10}             & 8.5            \\
scanalyzer                             & \multicolumn{1}{c|}{\textbf{100.0}} & \multicolumn{1}{c|}{65.3}         & \multicolumn{1}{c|}{\textbf{100.0}} & \multicolumn{1}{c|}{87}           & \multicolumn{1}{c|}{\textbf{100.0}} & \multicolumn{1}{c|}{\textbf{100.0}}  & \multicolumn{1}{c|}{\textbf{100.0}}  & \multicolumn{1}{c|}{\textbf{100.0}}  & \multicolumn{1}{c|}{\textbf{100.0}} & \multicolumn{1}{c|}{\textbf{100.0}} & \multicolumn{1}{c|}{\textbf{100.0}} & \multicolumn{1}{c|}{\textbf{100.0}} & \multicolumn{1}{c|}{\textbf{100.0}} & \multicolumn{1}{c|}{\textbf{100.0}} & \multicolumn{1}{c|}{\textbf{100.0}} & \textbf{100.0} \\
storage                                & \multicolumn{1}{c|}{0.0}            & \multicolumn{1}{c|}{0.0}            & \multicolumn{1}{c|}{12}           & \multicolumn{1}{c|}{0.0}            & \multicolumn{1}{c|}{24.5}         & \multicolumn{1}{c|}{7}             & \multicolumn{1}{c|}{24}            & \multicolumn{1}{c|}{0.0}             & \multicolumn{1}{c|}{1}            & \multicolumn{1}{c|}{0.0}            & \multicolumn{1}{c|}{7}            & \multicolumn{1}{c|}{0.0}            & \multicolumn{1}{c|}{8.0}            & \multicolumn{1}{c|}{0.0}            & \multicolumn{1}{c|}{29}             & \textbf{35.5}  \\
transport                              & \multicolumn{1}{c|}{8.2}          & \multicolumn{1}{c|}{0.0}            & \multicolumn{1}{c|}{11.5}         & \multicolumn{1}{c|}{42.8}         & \multicolumn{1}{c|}{98.8}         & \multicolumn{1}{c|}{97}            & \multicolumn{1}{c|}{99.2}          & \multicolumn{1}{c|}{\textbf{100.0}}  & \multicolumn{1}{c|}{0.0}            & \multicolumn{1}{c|}{0.0}            & \multicolumn{1}{c|}{1.8}          & \multicolumn{1}{c|}{32.8}         & \multicolumn{1}{c|}{49}             & \multicolumn{1}{c|}{70.8}           & \multicolumn{1}{c|}{62.8}           & 84.8           \\
visitall                               & \multicolumn{1}{c|}{0.0}            & \multicolumn{1}{c|}{13}           & \multicolumn{1}{c|}{4.3}          & \multicolumn{1}{c|}{94}           & \multicolumn{1}{c|}{86.7}         & \multicolumn{1}{c|}{92}            & \multicolumn{1}{c|}{92}            & \multicolumn{1}{c|}{\textbf{100.0}}  & \multicolumn{1}{c|}{0.0}            & \multicolumn{1}{c|}{5}            & \multicolumn{1}{c|}{1.7}          & \multicolumn{1}{c|}{79}           & \multicolumn{1}{c|}{17.3}           & \multicolumn{1}{c|}{80.7}           & \multicolumn{1}{c|}{26.7}           & 99             \\ \hline
\textbf{Average}                       & \multicolumn{1}{c|}{12.7}         & \multicolumn{1}{c|}{15.9}         & \multicolumn{1}{c|}{19.3}         & \multicolumn{1}{c|}{36.7}         & \multicolumn{1}{c|}{53.0}         & \multicolumn{1}{c|}{60.8}          & \multicolumn{1}{c|}{57.8}          & \multicolumn{1}{c|}{\textbf{60.9}} & \multicolumn{1}{c|}{13.1}         & \multicolumn{1}{c|}{19.7}         & \multicolumn{1}{c|}{14.7}         & \multicolumn{1}{c|}{34.7}         & \multicolumn{1}{c|}{35.4}           & \multicolumn{1}{c|}{46.4}           & \multicolumn{1}{c|}{43.0}           & 59.9           \\ \hline
\multicolumn{17}{c}{\textbf{Hard Tasks}}                                                                                                                                                                                                                                                                                                                                                                                                                                                                                                                                                                                     \\ \hline
blocks                                 & \multicolumn{1}{c|}{0.0}            & \multicolumn{1}{c|}{0.0}            & \multicolumn{1}{c|}{0.0}            & \multicolumn{1}{c|}{18.4}         & \multicolumn{1}{c|}{41.6}         & \multicolumn{1}{c|}{\textbf{44.4}} & \multicolumn{1}{c|}{34.4}          & \multicolumn{1}{c|}{34.4}          & \multicolumn{1}{c|}{0.0}            & \multicolumn{1}{c|}{0.0}            & \multicolumn{1}{c|}{0.0}            & \multicolumn{1}{c|}{12.8}         & \multicolumn{1}{c|}{27.6}           & \multicolumn{1}{c|}{42}             & \multicolumn{1}{c|}{24}             & 41.2           \\
depot                                  & \multicolumn{1}{c|}{0.0}            & \multicolumn{1}{c|}{0.0}            & \multicolumn{1}{c|}{0.0}            & \multicolumn{1}{c|}{0.0}            & \multicolumn{1}{c|}{6}            & \multicolumn{1}{c|}{8}             & \multicolumn{1}{c|}{\textbf{13.1}} & \multicolumn{1}{c|}{8.3}           & \multicolumn{1}{c|}{0.0}            & \multicolumn{1}{c|}{0.0}            & \multicolumn{1}{c|}{0.0}            & \multicolumn{1}{c|}{0.0}            & \multicolumn{1}{c|}{0.3}            & \multicolumn{1}{c|}{5.4}            & \multicolumn{1}{c|}{4.9}            & 10.6           \\
grid                                   & \multicolumn{1}{c|}{0.0}            & \multicolumn{1}{c|}{0.0}            & \multicolumn{1}{c|}{0.0}            & \multicolumn{1}{c|}{0.0}            & \multicolumn{1}{c|}{0.0}            & \multicolumn{1}{c|}{\textbf{8}}    & \multicolumn{1}{c|}{0.2}           & \multicolumn{1}{c|}{0.8}           & \multicolumn{1}{c|}{0.0}            & \multicolumn{1}{c|}{0.0}            & \multicolumn{1}{c|}{0.0}            & \multicolumn{1}{c|}{0.0}            & \multicolumn{1}{c|}{0.0}              & \multicolumn{1}{c|}{0.0}              & \multicolumn{1}{c|}{0.0}              & 5.2            \\
npuzzle                                & \multicolumn{1}{c|}{0.0}            & \multicolumn{1}{c|}{0.0}            & \multicolumn{1}{c|}{0.0}            & \multicolumn{1}{c|}{0.0}            & \multicolumn{1}{c|}{0.0}            & \multicolumn{1}{c|}{0.0}             & \multicolumn{1}{c|}{0.0}             & \multicolumn{1}{c|}{0.0}             & \multicolumn{1}{c|}{0.0}            & \multicolumn{1}{c|}{0.0}            & \multicolumn{1}{c|}{0.0}            & \multicolumn{1}{c|}{0.0}            & \multicolumn{1}{c|}{0.0}              & \multicolumn{1}{c|}{0.0}              & \multicolumn{1}{c|}{0.0}              & 0.0              \\
pipesworld-nt                          & \multicolumn{1}{c|}{0.1}          & \multicolumn{1}{c|}{0.1}          & \multicolumn{1}{c|}{0.8}          & \multicolumn{1}{c|}{4.5}          & \multicolumn{1}{c|}{5.3}          & \multicolumn{1}{c|}{8.6}           & \multicolumn{1}{c|}{5.5}           & \multicolumn{1}{c|}{12.1}          & \multicolumn{1}{c|}{2.2}          & \multicolumn{1}{c|}{5.4}          & \multicolumn{1}{c|}{4.4}          & \multicolumn{1}{c|}{8.1}          & \multicolumn{1}{c|}{4.8}            & \multicolumn{1}{c|}{\textbf{12.4}}  & \multicolumn{1}{c|}{8.8}            & 8.8            \\
rovers                                 & \multicolumn{1}{c|}{0.0}            & \multicolumn{1}{c|}{0.0}            & \multicolumn{1}{c|}{0.0}            & \multicolumn{1}{c|}{0.0}            & \multicolumn{1}{c|}{0.0}            & \multicolumn{1}{c|}{0.0}             & \multicolumn{1}{c|}{0.0}             & \multicolumn{1}{c|}{\textbf{0.1}}  & \multicolumn{1}{c|}{0.0}            & \multicolumn{1}{c|}{0.0}            & \multicolumn{1}{c|}{0.0}            & \multicolumn{1}{c|}{0.0}            & \multicolumn{1}{c|}{0.0}              & \multicolumn{1}{c|}{0.0}              & \multicolumn{1}{c|}{0.0}              & 0.0              \\
scanalyzer                             & \multicolumn{1}{c|}{48}           & \multicolumn{1}{c|}{34}           & \multicolumn{1}{c|}{83.3}         & \multicolumn{1}{c|}{66}           & \multicolumn{1}{c|}{\textbf{100.0}} & \multicolumn{1}{c|}{96.7}          & \multicolumn{1}{c|}{\textbf{100.0}}  & \multicolumn{1}{c|}{\textbf{100.0}}  & \multicolumn{1}{c|}{66.7}         & \multicolumn{1}{c|}{54}           & \multicolumn{1}{c|}{98.7}         & \multicolumn{1}{c|}{81.3}         & \multicolumn{1}{c|}{\textbf{100.0}}   & \multicolumn{1}{c|}{70.7}           & \multicolumn{1}{c|}{\textbf{100.0}}   & \textbf{100.0}   \\
storage                                & \multicolumn{1}{c|}{0.0}            & \multicolumn{1}{c|}{0.0}            & \multicolumn{1}{c|}{0.0}            & \multicolumn{1}{c|}{0.0}            & \multicolumn{1}{c|}{6.8}          & \multicolumn{1}{c|}{8}             & \multicolumn{1}{c|}{8.2}           & \multicolumn{1}{c|}{0.0}             & \multicolumn{1}{c|}{0.0}            & \multicolumn{1}{c|}{0.0}            & \multicolumn{1}{c|}{2.5}          & \multicolumn{1}{c|}{3}            & \multicolumn{1}{c|}{0.8}            & \multicolumn{1}{c|}{0.2}            & \multicolumn{1}{c|}{2}              & \textbf{12.5}  \\
transport                              & \multicolumn{1}{c|}{0.0}            & \multicolumn{1}{c|}{0.0}            & \multicolumn{1}{c|}{0.0}            & \multicolumn{1}{c|}{0.0}            & \multicolumn{1}{c|}{0.0}            & \multicolumn{1}{c|}{3.6}           & \multicolumn{1}{c|}{0.0}             & \multicolumn{1}{c|}{\textbf{14}}   & \multicolumn{1}{c|}{0.0}            & \multicolumn{1}{c|}{0.0}            & \multicolumn{1}{c|}{0.0}            & \multicolumn{1}{c|}{0.0}            & \multicolumn{1}{c|}{0.0}              & \multicolumn{1}{c|}{0.0}              & \multicolumn{1}{c|}{0.0}              & 0.8            \\
visitall                               & \multicolumn{1}{c|}{0.0}            & \multicolumn{1}{c|}{0.0}            & \multicolumn{1}{c|}{0.0}            & \multicolumn{1}{c|}{10}           & \multicolumn{1}{c|}{0.0}            & \multicolumn{1}{c|}{6}             & \multicolumn{1}{c|}{8}             & \multicolumn{1}{c|}{\textbf{100.0}}  & \multicolumn{1}{c|}{0.0}            & \multicolumn{1}{c|}{0.0}            & \multicolumn{1}{c|}{0.0}            & \multicolumn{1}{c|}{2}            & \multicolumn{1}{c|}{0.0}              & \multicolumn{1}{c|}{44}             & \multicolumn{1}{c|}{0.0}              & \textbf{100}   \\ \hline
\textbf{Average}                       & \multicolumn{1}{c|}{4.81}         & \multicolumn{1}{c|}{3.41}         & \multicolumn{1}{c|}{8.41}         & \multicolumn{1}{c|}{9.89}         & \multicolumn{1}{c|}{15.97}        & \multicolumn{1}{c|}{18.33}         & \multicolumn{1}{c|}{16.94}         & \multicolumn{1}{c|}{26.97}         & \multicolumn{1}{c|}{6.89}         & \multicolumn{1}{c|}{5.94}         & \multicolumn{1}{c|}{10.56}        & \multicolumn{1}{c|}{10.72}        & \multicolumn{1}{c|}{13.35}          & \multicolumn{1}{c|}{17.47}          & \multicolumn{1}{c|}{13.97}          & \textbf{27.91}
\end{tabular}
    \vspace{-0.07in}
    \caption{Comparison of the coverage given a 6 minute planning time budget of $h$\textsuperscript{RSL} using a range of different hyper-parameter values. Each run uses only a single trial and is evaluated over 10 of the 50 initial states for each instance in the benchmark set.}
    \label{tab:grid_baselines}
\end{table*}

\end{document}